\documentclass[10pt,twocolumn,letterpaper]{article}

\usepackage{iccv}
\usepackage{times}
\usepackage{epsfig}
\usepackage{graphicx}
\usepackage{amsmath}
\usepackage{amssymb}

\usepackage[utf8]{inputenc} 
\usepackage[T1]{fontenc}    
\usepackage{url}            
\usepackage{booktabs}       
\usepackage{amsfonts}       
\usepackage{nicefrac}       
\usepackage{microtype}      
\usepackage{comment}
\usepackage{subfigure} 
\usepackage{pdflscape}
\usepackage{algorithm}
\usepackage{algorithmic}
\usepackage{bm}
\usepackage{soul}
\usepackage{listings}
\usepackage{bbm}
\usepackage{multirow}
\usepackage{multicol}
\usepackage{amsthm}

\newcommand\Tstrut{\rule{0pt}{2.2ex}}         

\newcommand{\keypoint}[1]{\noindent\textbf{#1}\quad}

\newtheorem{theorem}{Theorem}[section]

\usepackage[pagebackref=true,breaklinks=true,letterpaper=true,colorlinks,bookmarks=false]{hyperref}

\iccvfinalcopy 

\def\iccvPaperID{11093} 

\ificcvfinal\fi

\begin{document}

\title{A Hierarchical Bayesian Model for Deep Few-Shot Meta Learning}

\author{Minyoung Kim$^1$\\
$^1$Samsung AI Center Cambridge, UK\\
{\tt\small mikim21@gmail.com}
\and
Timothy Hospedales$^{1,2}$\\
$^2$University of Edinburgh, UK\\
{\tt\small t.hospedales@ed.ac.uk}
}

\maketitle
\ificcvfinal\thispagestyle{empty}\fi

\begin{abstract}
We propose a novel hierarchical Bayesian model for learning with a large (possibly infinite) number of tasks/episodes, which suits well the few-shot meta learning problem. We consider episode-wise random variables to model episode-specific target generative processes, where these local random variables are governed by a higher-level global random variate. The global variable helps memorize the important information from historic episodes while controlling how much the model needs to be adapted to new episodes in a principled Bayesian manner. Within our model framework, the prediction on a novel episode/task can be seen as a Bayesian inference problem. 
However, a main obstacle in learning with a large/infinite number of local random variables in online nature, 
is that one is not allowed to store the posterior distribution of the current local random variable for frequent future updates, typical in conventional variational inference. 
We need to be able to treat each local variable as a {\em one-time} iterate in the optimization. We propose a Normal-Inverse-Wishart model, for which we show that this one-time iterate optimization becomes feasible due to the approximate closed-form solutions for the local posterior distributions. The resulting algorithm is more attractive than the MAML in that it is not required to maintain computational graphs for the whole gradient optimization steps per episode. Our approach is also different from existing Bayesian meta learning methods in that unlike dealing with a single random variable for the whole episodes, our approach has a hierarchical structure that allows one-time episodic optimization, desirable for principled Bayesian learning with many/infinite tasks. 
The code is available at \url{https://github.com/minyoungkim21/niwmeta}.
\end{abstract}

\section{Introduction}\label{sec:intro}
Few-shot learning (FSL) aims to emulate the human ability to learn from few examples \cite{lake2015ppi}. It has received substantial and growing interest \cite{wang2020fslSurvey} due to the need to alleviate the notoriously data intensive nature of mainstream supervised deep learning.  Approaches to FSL are all based on some kind of knowledge transfer from a set of plentiful source recognition problems 
to the sparse data target problem of interest. Existing approaches are differentiated in terms of the assumptions they make about what is task agnostic knowledge that can be transferred from the source tasks, and what is task-specific knowledge that should be learned from the sparse target examples. For example, the seminal MAML \cite{maml} and ProtoNets \cite{protonet} respectively assume that the initialization for fine-tuning, or the feature extractor for metric-based recognition should be transferred from source categories.

One of the most principled and systematic ways to model such sets of related problems are hierarchical Bayesian models (HBMs) \cite{gelman2003bda}. The HBM paradigm is widely used in statistics, but has seen relatively less use in deep learning and computer vision, due to the technical difficulty of bringing hierarchical Bayesian modelling to bear on deep learning.  HBMs provide a powerful way to model a set of related problems, by assuming that each problem has its own parameters (e.g, the neural networks that recognize cat vs dog, or car vs bike), but that those problems share a common prior (the prior over such neural networks). Data-efficient learning of the target tasks is then achieved by inferring the prior based on the source tasks, and using it to enhance learning the posterior over the target task parameters. 

A Bayesian learning treatment of FSL would be appealing due to the overfitting resistance provided by Bayesian Occam’s razor \cite{mackay2003book}, as well as the ability to improve calibration of inference  so that the model’s confidence is reflective of its probability of correctness — a crucial property in mission critical applications \cite{ece}. However the limited attempts that have been made to exploit these tools in deep learning have either been incomplete treatments that only model a single Bayesian layer within the neural network \cite{metaqda,gordon2019metaPred}, or else fail to scale up to modern neural architectures \cite{platipus,bmaml}. 


In this paper we present the first complete hierarchical Bayesian learning algorithm for few-shot deep learning. Our algorithm efficiently learns a prior\footnote{
Precisely speaking, we have a higher-level random variable $\phi$ shared across episodes, and {\em learning a prior} means inferring the posterior $\phi|\{D_i\}$ for all episodic training data $\{D_i\}$. At test time, this posterior serves as a prior for generating network weights $\theta$ that is specific to each test episode. 
} over neural networks during the meta-train phase, and efficiently learns a posterior neural network during each meta-test episode. Importantly, our learning is architecture independent. It can scale up to state of the art backbones including ViTs~\cite{dosovitskiy2020vit}, and works smoothly with any few-shot learning architecture -- spanning simple linear decoders \cite{maml,protonet}, to those based on sophisticated set-based decoders such as FEAT \cite{feat} and CNP\cite{cnp}/ANP\cite{anp}.  We show empirically that our HBM provides improved performance and calibration in all of these cases, as well as providing clear theoretical justification. 

Our analysis also reveals novel links between seminal FSL methods such as ProtoNet~\cite{protonet}, MAML~\cite{maml}, and Reptile~\cite{reptile}, 
all of which are different special cases of our framework despite their very different appearance. Interestingly, despite its close relatedness to MAML-family algorithms, our Bayesian learner admits an efficient closed-form solution to the task-specific and task-agnostic updates that does not require maintaining the computational graph for reverse-mode backpropagation. This provides a novel solution to a famous meta-learning scalability bottleneck. 

In summary, our contributions include: (i) The first complete hierarchical Bayesian treatment of the few-shot deep learning problem, and associated theoretical justification. (ii) An efficient algorithmic learning solution that can scale up to modern architectures, and plug into most existing neural FSL meta-learners. (iii) Empirical results demonstrating improved accuracy and calibration performance on both classification and regression benchmarks.


\section{Problem Setup}\label{sec:setup}


We consider the {\em episodic few-shot learning} problem, which can be formally stated as follows. Let $p(\mathcal{T})$ be the (unknown) task/episode distribution, where each task $\mathcal{T} \sim p(\mathcal{T})$ is defined as a distribution $p_\mathcal{T}(x,y)$ for data $(x,y)$ where $x$ is input and $y$ is target. By episodic learning, we have a large (possibly infinite) number of episodes during training, $\mathcal{T}_1,\mathcal{T}_2,\dots, \mathcal{T}_N \sim P(\mathcal{T})$ sampled i.i.d., but we only observe a small number of labeled samples from each episode, denoted by 
$D_i\!=\!\{(x^i_j,y^i_j)\}_{j=1}^{n_i} \sim p_{\mathcal{T}_i}(x,y)$, where $n_i\!=\!|D_i|$ is the number of samples in $D_i$. 
The goal of the learner, after observing the training data $D_1,\dots,D_N$ from a large number of different tasks, is to build a predictor $p^*(y|x)$ for novel unseen tasks $\mathcal{T}^*\sim p(\mathcal{T})$. We will often abuse the notation, e.g., $i\sim\mathcal{T}$ refers to the episode $i$ sampled, i.e., $D_i\sim p_{\mathcal{T}_i}(x,y)$ where $\mathcal{T}_i \sim p(\mathcal{T})$.
%
At the test time we are allowed to have some hints about the new test task $\mathcal{T}^*$, in the form of a few labeled examples from $\mathcal{T}^*$, also known as the {\em support set}\footnote{
For the episodic training data $D_i$, it is common practice to partition 
it into two labeled sets, {\em support} and {\em query}, so that we use the support set for adaptation while measuring the quality of the adapted model on the query set to get learning signals. However, we do not explicitly deal with this convention in our derivations, but treat $D_i$ as a whole available training set. 
} denoted by $D^*\sim P_{\mathcal{T}^*}(x,y)$. 

For ease of exposition and theoretical analysis, we consider {\em infinite} episodes ($N\!\to\!\infty$) observed during training (of course in practice $N$ is large but finite). In Bayesian perspective, the goal is to infer the posterior distribution with the large/infinite number of episodic training data as evidence, that is, $p(y|x,D_{1:N})|_{N\to\infty}$. A major computational challenge is that the large/infinite number of tasks/data cannot be stored, hardly replayed or revisited, which implies that any viable learning algorithm has to be {\em online} 
in nature.

\section{Main Approach}\label{sec:main}


\begin{figure}
\vspace{-0.5em}
\centering
\begin{tabular}{ccc}
\includegraphics[trim = 9mm 2mm 0mm 0mm, clip, scale=0.145]{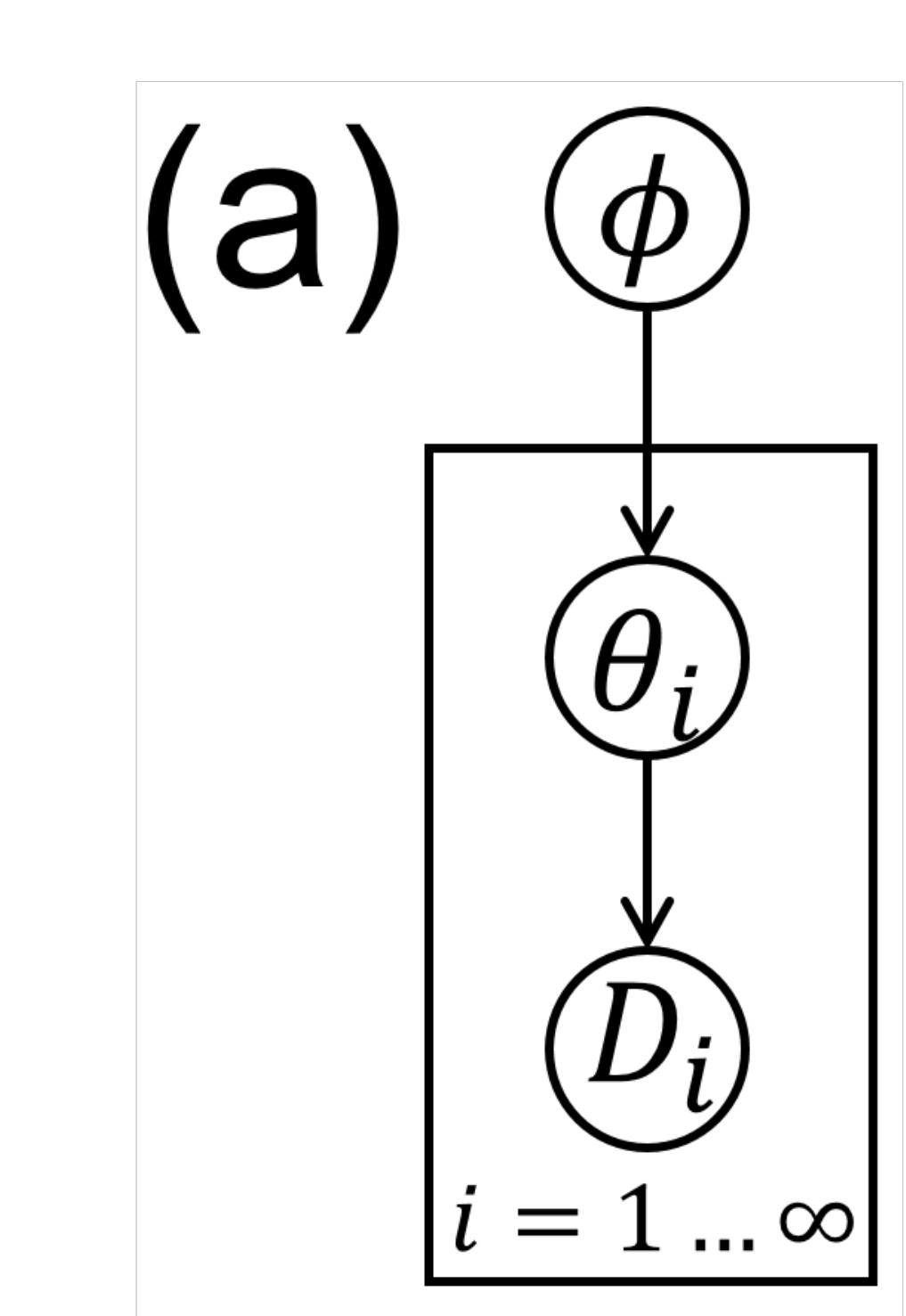} & 
\includegraphics[trim = 1mm 2mm 0mm 0mm, clip, scale=0.145]{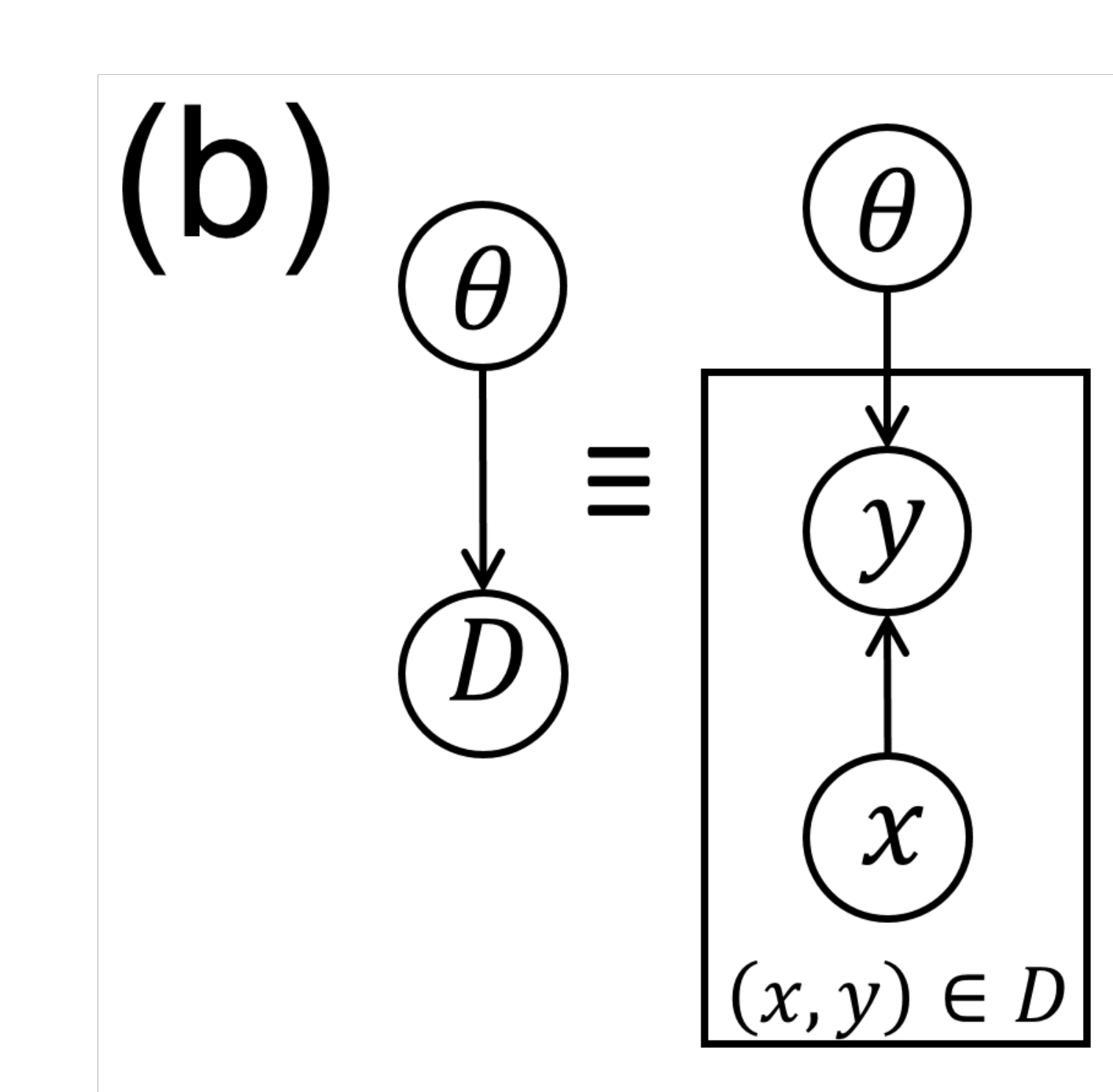} & 
\includegraphics[trim = 1mm 0mm 0mm 0mm, clip, scale=0.145]{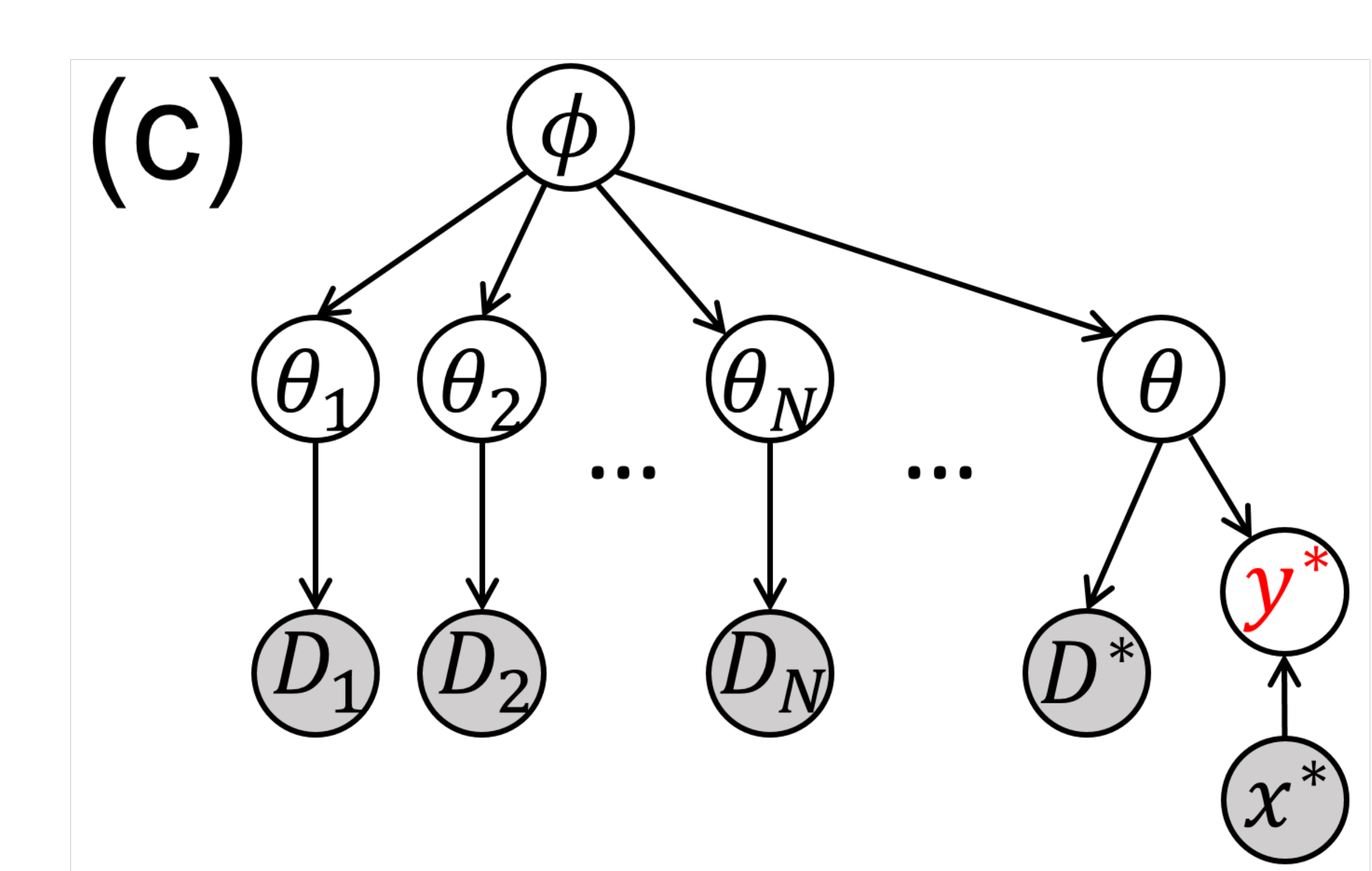}
\vspace{-0.2em}
\end{tabular}
\caption{Graphical models. (a) Plate view of iid episodes. (b) Individual episode data with input $x$ given and only $p(y|x)$ modeled. (c): Few-shot learning as a probabilistic inference problem (shaded nodes $=$ {\em evidences}, red colored nodes $=$ {\em targets} to infer). In (c), $D^*$ denotes the support set for the test episode. Note: a large number of (possibly infinitely many) evidences $D_1,D_2,\dots,D_N,\dots$.
}
\label{fig:gm}
\vspace{-1.0em}
\end{figure}

We introduce two types of latent random variables, $\phi$ and $\{\theta_i\}_{i=1}^\infty$. Each $\theta_i$, one for each episode $i$, 
is deployed as the network weights for modeling the data $D_i$ ($i=1,\dots,\infty$). Specifically, $D_i$ is generated\footnote{Note that 
we do not deal with generative modeling of input $x$. Inputs $x$ are always given, and only conditionals $p(y|x)$ are modeled (Fig.~\ref{fig:gm}(b)). 
} by $\theta_i$ as in the likelihood model in (\ref{eq:prior_lik}). 
The variable $\phi$ can be viewed as a globally shared variable that is responsible for linking the individual episode-wise parameters $\theta_i$. We assume conditionally independent and identical priors, $p(\{\theta_i\}_i|\phi) = \prod_i p(\theta_i|\phi)$.  
Thus the prior for the latent variables $(\phi,\{\theta_i\}_{i=1}^\infty)$ is formed in a hierarchical manner. 
The model is fully described as:
\begin{align}
&\textrm{(Prior)} \ \ \ \ 
p(\phi, \theta_{1:\infty}) = p(\phi) 
{\textstyle\prod}_{i=1}^\infty p(\theta_i|\phi) \\
&\textrm{(Likelihood)} \ \ \ \ p(D_i|\theta_i) = 
{\textstyle\prod}_{(x,y)\in D_i} p(y|x,\theta_i)
\label{eq:prior_lik}
\end{align}
where $p(y|x,\theta_i)$ is a conventional neural network model. 
See the graphical model in Fig.~\ref{fig:gm}(a) where the iid episodes are governed by a single random variable $\phi$.



Given infinitely many episodic data $\{D_i\}_{i=1}^\infty$ 
we infer the posterior, $p(\phi,\theta_{1:\infty}|D_{1:\infty}) \propto p(\phi) \prod_{i=1}^\infty p(\theta_i|\phi) p(D_i|\theta_i)$, 
and we adopt  variational inference to approximate it. That is, $q(\phi,\theta_{1:\infty}; L) \approx p(\phi,\theta_{1:\infty}|D_{1:\infty})$ where
\begin{align}
q(\phi,\theta_{1:\infty}; L) 
:= q(\phi; L_0) \cdot \lim_{N\to\infty} 
{\textstyle\prod}_{i=1}^N q_i(\theta_i; L_i),
\label{eq:q_1}
\end{align}
where the variational parameters $L$ consists of $L_0$ (parameters for $q(\phi)$) and $\{L_i\}_{i=1}^\infty$'s (parameters of  $q_i(\theta_i)$'s for episode $i$). 
Note that although $\theta_i$'s are independent across episodes under (\ref{eq:q_1}), they are differently modeled (note the subscript $i$ in notation $q_i$), reflecting different posterior beliefs originating from heterogeneity of episodic data $D_i$'s.

\keypoint{Normal-Inverse-Wishart model.} 
We consider Normal-Inverse-Wishart (NIW) distributions for the prior and variational posterior. First, the prior is modeled as a conjugate form of Gaussian and NIW. With $\phi = (\mu,\Sigma)$,
\begin{align}
&p(\phi) = 
\mathcal{N}(\mu; \mu_0, \lambda_0^{-1}\Sigma) \cdot  \mathcal{IW}(\Sigma; \Sigma_0, \nu_0), \label{eq:niw_prior_phi} \\
&p(\theta_i|\phi) = \mathcal{N}(\theta_i; \mu, \Sigma), \ \ i=1,\dots,\infty, \label{eq:niw_prior_theta}
\end{align}
where $\Lambda=\{\mu_0,\Sigma_0,\lambda_0,\nu_0\}$ is the parameters of the NIW. We do not need to pay attention to the choice of values for $\Lambda$ since $p(\phi)$ has vanishing effect on posterior due to the large/infinite number of evidences as we will see shortly.  
Next, our choice of the variational density family for $q(\phi)$ is the NIW, mainly because it admits closed-form expressions in the ELBO function due to the conjugacy, allowing one-time episodic optimization, as will be shown. 
\begin{align}
q(\phi; L_0) := 
\mathcal{N}(\mu; m_0, l_0^{-1}\Sigma) \cdot  \mathcal{IW}(\Sigma; V_0, n_0).
\label{eq:niw_q_phi}
\end{align}
So, $L_0 = \{m_0,V_0,l_0,n_0\}$, and we restrict $V_0$ to be diagonal. 
The density family for $q_i(\theta_i)$'s is chosen as a Gaussian,
\begin{align}
q_i(\theta_i; L_i) = \mathcal{N}(\theta_i; m_i, V_i).
\label{eq:niw_q_th}
\end{align}
Thus $L_i=\{m_i,V_i\}$. Learning (variational inference) amounts to finding $L_0$ and $\{L_i\}_{1}^\infty$ that makes the 
approximation $q(\phi,\theta_{1:\infty}; L) \approx p(\phi,\theta_{1:\infty}|D_{1:\infty})$, as tight as possible.

\keypoint{Variational inference.} 
For the finite case with $N$ episodes, it is straightforward to derive the upper bound of the negative marginal log-likelihood (NMLL) as 
\begin{align}
&-\log p(D_{1:N}) \ \leq \  
\textrm{KL}(q(\phi)||p(\phi)) \ + \label{eq:elbo_finite} \\
& \ \ \ \ \ \ {\textstyle\sum}_{i=1}^N \Big( \mathbb{E}_{q_i(\theta_i)}[l_i(\theta_i)] + \mathbb{E}_{q(\phi)}\big[\textrm{KL}(q_i(\theta_i) || p(\theta_i|\phi))\big] \Big) \nonumber
\end{align}
where $l_i(\theta_i)\!=\!-\log p(D_i|\theta_i)$ is the negative training log-likelihood 
of $\theta_i$ in episode $i$.
As $N\!\to\!\infty$, the ultimate objective that we like to minimize is naturally the {\em effective episode-averaged NMLL}, that is,  $\lim_{N\to\infty}-\frac{1}{N} \log p(D_{1:N})$, whose bound is derived from (\ref{eq:elbo_finite}) as:
\begin{align}
\lim_{N\to\infty} \frac{1}{N}\!{\textstyle\sum}_{i=1}^N\!\Big(\!\mathbb{E}_{q_i(\theta_i)}[l_i(\theta_i)]\!+\!\mathbb{E}_{q(\phi)}\!\big[\textrm{KL}(q_i(\theta_i) || p(\theta_i|\phi))\big]\!\Big)
\nonumber 
\end{align}
Note that $\frac{1}{N}\textrm{KL}(q(\phi)||p(\phi))$ vanished as $N\!\to\!\infty$.
Since $\lim_{N\to\infty} \frac{1}{N} \sum_{i=1}^N f_i = \mathbb{E}_{i\sim\mathcal{T}}[f_i]$ for any expression $f_i$, the ELBO learning amounts to the following optimization:
\begin{align}
\min_{L_0, \{L_i\}_{i=1}^\infty} & \mathbb{E}_{i\sim\mathcal{T}} \Big[ \ \mathbb{E}_{q_i(\theta_i;L_i)}[l_i(\theta_i)] \ \ + \label{eq:elbo_optim_orig} \\ 
& \ \ \ \ \ \ \ \ \ \ \ \ \ \ \mathbb{E}_{q(\phi;L_0)}\big[\textrm{KL}(q_i(\theta_i;L_i) || p(\theta_i|\phi))\big] \ \Big]. \nonumber
\end{align}

\keypoint{One-time episodic optimization.} 
Note that (\ref{eq:elbo_optim_orig}) is challenging due to the large/infinite number of optimization variables $\{L_i\}_{i=1}^\infty$ and the online nature of task sampling $i\sim\mathcal{T}$. Applying conventional SGD would simply fail since each $L_i$ will never be updated more than once. Instead, we tackle it by finding the optimal solutions for $L_i$'s for fixed $L_0$, thus effectively representing the optimal solutions as functions of $L_0$, namely $\{L_i^*(L_0)\}_{i=1}^\infty$. Plugging the optimal $L_i^*(L_0)$'s back to (\ref{eq:elbo_optim_orig}) leads to the optimization problem over $L_0$ alone. 
The idea is just like solving: $\min_{x,y} f(x,y) = \min_x f(x, y^*(x))$ where $y^*(x) = \arg\min_y f(x,y)$ with $x$ fixed.

Note that when we fix $L_0$ (i.e., fix $q(\phi)$), the objective (\ref{eq:elbo_optim_orig}) is completely separable over $i$, and we can optimize individual $i$ independently. More specifically, for each $i\geq 1$, 
\begin{align}
\min_{L_i} \mathbb{E}_{q_i(\theta_i;L_i)}[l_i(\theta_i)] + \mathbb{E}_\phi \big[\textrm{KL}(q_i(\theta_i;L_i) || p(\theta_i|\phi))\big]
\label{eq:elbo_optim_indiv_i}
\end{align}
As the expected KL term in (\ref{eq:elbo_optim_indiv_i}) admits a closed form due to NIW-Gaussian conjugacy (Supp.~for derivations), 
we can reduce (\ref{eq:elbo_optim_indiv_i}) to the following optimization for $L_i=(m_i,V_i)$:
\begin{align}
& L_i^*(L_0) := \arg\min_{m_i,V_i} \ \bigg( \mathbb{E}_{\mathcal{N}(\theta_i;m_i,V_i)}[l_i(\theta_i)] - \frac{1}{2} \log |V_i| \ + \nonumber \\ 
& \ \ \ \frac{n_0}{2} (m_i\!-\!m_0)^\top V_0^{-1} (m_i\!-\!m_0) + \frac{n_0}{2} \textrm{Tr}\big(V_i V_0^{-1}\big) \bigg),
\label{eq:elbo_optim_miVi}
\end{align}
with $L_0=\{m_0,V_0,l_0,n_0\}$ fixed.

\keypoint{Quadratic approximation of episodic loss via SGLD.}
To find the closed-form solution $L_i^*(L_0)$ in (\ref{eq:elbo_optim_miVi}), we make quadratic approximation of $l_i(\theta_i)= -\!\log p(D_i|\theta_i)$. In general, $-\!\log p(D_i|\theta)$, as a function of $\theta$, can be written as:
\begin{align}
-\!\log p(D_i|\theta) \approx \frac{1}{2}(\theta\!-\!\overline{m}_i)^\top \overline{A}_i(\theta\!-\!\overline{m}_i) + \textrm{const.},
\label{eq:quad_approx}
\end{align}
for some $(\overline{m}_i, \overline{A}_i)$ that are constant with respect to $\theta$. 
One may attempt to obtain $(\overline{m}_i, \overline{A}_i)$ via Laplace approximation (e.g., the minimizer of $-\!\log p(D_i|\theta)$ for $\overline{m}_i$ and the Hessian at the minimizer for $\overline{A}_i$). 
However, this involves computationally intensive Hessian computation. Instead, using the fact that the log-posterior $\log p(\theta|D_i)$ equals (up to constant) $\log p(D_i|\theta)$ when we use uninformative prior $p(\theta)\propto 1$, we can obtain samples from the posterior $p(\theta|D_i)$ using MCMC sampling, especially the stochastic gradient Langevin dynamics (SGLD)~\cite{sgld}, and estimate sample mean and precision, which become $\overline{m}_i$ and $\overline{A}_i$, respectively\footnote{This approach is algorithmically very similar to the stochastic weight averaging (SWA)~\cite{swa} 
and follow-up Gaussian fitting (SWAG)~\cite{swag}. 
}. Note that this amounts to performing several SGD iterations (skipping a few initial for burn-in), and unlike MAML~\cite{maml} no computation graph needs to be maintained since $(\overline{m}_i, \overline{A}_i)$ are constant.
Once we have $(\overline{m}_i, \overline{A}_i)$, the optimization (\ref{eq:elbo_optim_miVi}) admits the closed-form solution (Supplement for derivations),
\begin{align}
m_i^*(L_0) &= (\overline{A}_i + n_0 V_0^{-1})^{-1} (\overline{A}_i \overline{m}_i + n_0 V_0^{-1} m_0), \nonumber \\ 
V_i^*(L_0) &= (\overline{A}_i + n_0 V_0^{-1})^{-1}.
\label{eq:miVi_star}
\end{align}
Computation in (\ref{eq:miVi_star}) is cheap since all matrices are diagonal.

\keypoint{Final optimization.} 
Plugging (\ref{eq:miVi_star}) back to (\ref{eq:elbo_optim_orig}), we have an optimization problem over 
$L_0\!=\!\{m_0,V_0,l_0,n_0\}$ alone, which can be 
written as (Supplement for full derivations): 
\begin{align}
&\min_{L_0} \ \mathbb{E}_{i\sim\mathcal{T}} \Big[ f_i(L_0) + \frac{1}{2} g_i(L_0) + \frac{d}{2 l_0} \Big] \ \ \textrm{s.t.} \label{eq:ultimate_optim} \\
& \ \ \ f_i(L_0) \ = \ \mathbb{E}_{\epsilon\sim \mathcal{N}(0,I)}\Big[l_i\Big(m_i^*(L_0)+V_i^*(L_0)^{1/2}\epsilon\Big)\Big], \nonumber \\ 
& \ \ \ g_i(L_0) \ = \ \log \frac{|V_0|}{|V_i^*(L_0)|} + n_0 \textrm{Tr}\big(V_i^*(L_0)V_0^{-1}\big) \ + \nonumber \\ 
& \ \ \ \ \ \ \ n_0 \big(m_i^*(L_0)\!-\!m_0\big)^\top V_0^{-1} \big(m_i^*(L_0)\!-\!m_0\big) - \psi_d\Big(\frac{n_0}{2}\Big), \nonumber
\end{align}
where $\psi_d(\cdot)$ is the multivariate digamma function and $d\!=\!\dim(\theta)$. 
As $l_0$ only appears in the term $\frac{d}{2l_0}$, the optimal value is $l_0^*\!=\!\infty$\footnote{This is compatible with the conjugate Gaussian observation case, where the posterior NIW has $l_0$ incremented from the prior's $l_0$ by the number of observations, which is $\infty$ in our case.}. 
We use SGD to solve (\ref{eq:ultimate_optim}), repeating the steps:
\begin{align}
\textrm{1) Sample} \ i\!\sim\!\mathcal{T}. \ \ \textrm{2)} \ L_0 \leftarrow L_0\!-\!\eta \nabla_{L_0} \Big( f_i(L_0)\!+\!\frac{1}{2} g_i(L_0) \Big). \nonumber
\end{align}
Note that $\nabla_{L_0} \big( f_i(L_0) + \frac{1}{2} g_i(L_0) \big)$ is an {\em unbiased} stochastic estimate for the gradient of the objective $\mathbb{E}_{i \sim \mathcal{T}} [\cdots]$ in (\ref{eq:ultimate_optim}). Furthermore, our  learning algorithm above (also pseudocode in Alg~\ref{alg:main}) is fully compatible with the online nature of the episodic training. 
After training, we obtain the learned $L_0$, that is, the posterior $q(\phi; L_0)$. The learned posterior $q(\phi; L_0)$ will be used at the meta test time, 
where we show in Sec.~\ref{sec:meta_test} that this can be seen as Bayesian inference as well. 

We emphasize that our framework is completely flexible in the choice of the backbone $p(y|x,\theta)$. It could be the popular instance-based network comprised of a feature extractor and a prediction head where the latter can be either a conventional learnable readout head or the parameter-free one like the nearest centroid classifier (NCC) in ProtoNet~\cite{protonet}, i.e., $p(D|\theta)\!=\!p(Q|S,\theta)$ where $D\!=\!S \cup Q$ and $p(y|x,S,\theta)$ is the NCC prediction with support $S$. We can also adopt the set-based networks~\cite{feat,cnp,anp} where $p(y|x,S,\theta)$ itself is modeled by a neural net $y = G(x,S;\theta)$ with input $(x,S)$.

\newcommand\inlineeqno{\stepcounter{equation}\ (\theequation)}
\begin{algorithm}[t!]
  \caption{Our few-shot meta learning algorithm.}
  \label{alg:main}
\begin{small}
\begin{algorithmic}
   \STATE {\bfseries Initialize:} $L_0=\{m_0,V_0,n_0\}$ of $q(\phi;L_0)$ randomly. 
   \FOR{episode $i=1,2,\dots$}
      \STATE Perform SGLD iterations on $D_i$ to estimate $(\overline{m}_i,\overline{A}_i)$.
      \STATE Compute the episodic minimizer $L_i^*(L_0)$ from (\ref{eq:miVi_star}).
      \STATE Update $L_0$ by the gradient of $f_i(L_0) + \frac{1}{2} g_i(L_0)$ as in (\ref{eq:ultimate_optim}).
   \ENDFOR
   \STATE {\bfseries Output:} Learned $L_0$.
\end{algorithmic}
\end{small}
\end{algorithm}

\subsection{Interpretation}\label{sec:interpretation}


We show that our framework unifies seemingly unrelated seminal FSL algorithms into one perspective. 


\keypoint{MAML~\cite{maml} as a special case.} Suppose we consider spiky variational densities, i.e., $V_i\!\to\!0$ (constant). The one-time episodic optimization (\ref{eq:elbo_optim_miVi}) reduces to: $\arg\min_{m_i} l_i(\theta_i)\!+\!R(m_i)$ where $R(m_i)$ is the quadratic penalty of $m_i$ deviating from $m_0$. 
One reasonable solution is to perform a few gradient steps with loss $l_i$, starting from $m_0$ to have small penalty ($R\!=\!0$ initially). That is, $m_i\!\leftarrow\!m_0$ and a few steps of $m_i \leftarrow m_i - \alpha \nabla l_i(m_i)$ to return $m_i^*(L_0)$. Plugging this into (\ref{eq:ultimate_optim}) while disregarding the $g_i$ term, leads to the MAML algorithm. Obviously, the main drawback is $m_i^*(L_0)$ is a function of $m_0 \in L_0$ via a full computation graph of SGD steps, compared to our lightweight closed forms (\ref{eq:miVi_star}).

\keypoint{ProtoNet~\cite{protonet} as a special case.} Again with $V_i\!\to\!0$, if we ignore the negative log-likelihood term in (\ref{eq:elbo_optim_miVi}), then the optimal solution becomes $m_i^*(L_0) = m_0$. If we remove the $g_i$ term, we can solve (\ref{eq:ultimate_optim}) by simple gradient descent with $\nabla_{m_0}(-\log p(D_i|m_0))$. We then adopt 
the NCC head and regard $m_0$ as sole feature extractor parameters, 
which becomes exactly the ProtoNet update. 

\keypoint{Reptile~\cite{reptile} as a special case.} 
Instead, if we ignore all penalty terms in (\ref{eq:elbo_optim_miVi}) and follow our quadratic approximation (\ref{eq:quad_approx}) with $V_i\!\to\!0$, then $m_i^*(L_0) = \overline{m}_i$.  It is constant with respect to $L_0 = (m_0,V_0,n_0)$, and makes the optimization (\ref{eq:ultimate_optim}) very simple: the optimal $m_0$ is the average of $\overline{m}_i$ for all tasks $i$, i.e., $m_0^* = \mathbb{E}_{i \sim \mathcal{T}}[\overline{m}_i]$ (we ignore $V_0$ here). Note that Reptile ultimately finds the exponential smoothing of $m_i^{(k)}$ over $i\sim \mathcal{T}$ where $m_i^{(k)}$ is the iterate after $k$ SGD steps for task $i$. This can be seen as an online estimate of $\mathbb{E}_{i \sim \mathcal{T}}[\overline{m}_i]$. 

\subsection{Meta Test Prediction as Bayesian Inference}\label{sec:meta_test}

At meta test time, we need to be able to predict the target $y^*$ of a novel test input $x^*\sim \mathcal{T}^*$ sampled from the unknown distribution $\mathcal{T}^*\sim p(\mathcal{T})$.  In FSL, 
we have 
the test support data $D^*=\{(x,y)\}\sim \mathcal{T}^*$. 
%
The test-time prediction can be seen as a posterior inference problem with {\em additional evidence} of the support data $D^*$ (Fig.~\ref{fig:gm}(c)). More specifically, 
\begin{align}
p(y^*|x^*, D^*, D_{1:\infty}) = \int p(y^*|x^*, \theta) \ p(\theta|D^*, D_{1:\infty}) \ d\theta.
\nonumber
\end{align}
So, it boils down to $p(\theta|D^*,D_{1:\infty})$, the posterior given both the test support data $D^*$ and the entire training data $D_{1:\infty}$. 
Under our hierarchical model, exploiting conditional independence (Fig.~\ref{fig:gm}(c)), we can link it to our trained $q(\phi)$ as: 
\begin{align}
&p(\theta|D^*,D_{1:\infty}) 
\approx \int p(\theta|D^*,\phi) \ p(\phi|D_{1:\infty}) \ d\phi \label{eq:personal_1} \\
& \ \ \ \ \ \ \ \ \approx \int p(\theta|D^*,\phi) \ q(\phi) \ d\phi 
\ \approx \ p(\theta|D^*,\phi^*), \label{eq:personal_3}
\end{align}
where in (\ref{eq:personal_1}) we disregard the impact of $D^*$ on the higher-level $\phi$ given the joint evidence, i.e., $p(\phi|D^*,D_{1:\infty})\approx p(\phi|D_{1:\infty})$, due to dominance of $D_{1:\infty}$ compared to smaller $D^*$. 
The last part of (\ref{eq:personal_3}) makes approximation using the mode $\phi^*$ of $q(\phi)$, 
where $\phi^*=(\mu^*,\Sigma^*)$ has a closed form: 
\begin{align}
\mu^* = m_0, \ \ \ \ \Sigma^* = \frac{V_0}{n_0+d+2}.
\label{eq:niw_phi*}
\end{align}
%
Next, since $p(\theta|D^*,\phi^*)$ 
involves difficult marginalization $p(D^*|\phi^*) = \int p(D^*|\theta) p(\theta|\phi^*) d\theta$, we adopt  variational inference, introducing a tractable variational distribution $v(\theta) \approx p(\theta|D^*,\phi^*)$. With the Gaussian family as in the training time (\ref{eq:niw_q_th}), i.e., $v(\theta) = \mathcal{N}(\theta; m, V)$ where $(m,V)$ are the variational parameters optimized by ELBO optimization, 
\begin{align}
\min_{m,V} \  \mathbb{E}_{v(\theta)}[-\log p(D^*|\theta)] + \textrm{KL}(v(\theta)||p(\theta|\phi^*)).
\label{eq:test_elbo}
\end{align}
See Supplement for detailed formulas for (\ref{eq:test_elbo}). 
Once we have the optimized model $v$, our predictive distribution becomes:
\begin{align}
p(y^*|x^*,D^*,D_{1:\infty}) \approx 
\frac{1}{S} \sum_{s=1}^{M_S} p(y^*|x^*,\theta^{(s)}), 
\ \ \theta^{(s)} \sim v(\theta),
\nonumber
\end{align}
which simply requires feed-forwarding $x^*$ through the sampled networks $\theta^{(s)}$ and averaging.
Our meta-test algorithm is also summarized in the Supplementary Material. 
Note that we have test-time backbone update as per (\ref{eq:test_elbo}), which can make the final $m$ deviated from the learned mean $m_0$. Alternatively, if we drop the first term in (\ref{eq:test_elbo}), the optimal $v(\theta)$ equals $p(\theta|\phi^*)=\mathcal{N}(\theta; m_0, V_0/(n_0+d+2))$. This can be seen as using the learned model $m_0$ with some small random perturbation as a test-time backbone $\theta$.


\section{Theoretical Analysis}\label{sec:theorem}



\keypoint{Generalization error bounds.} 
We offer two theorems that upper-bound the generalization error of the model that is averaged over the learned posterior 
$q(\phi,\theta_{1:\infty})$. 
The first theorem relates the generalization error to the ultimate ELBO loss (\ref{eq:elbo_optim_orig}) that we minimized in our algorithm. We do this by utilizing the recent PAC-Bayes-$\lambda$ bound~\cite{pac_bayes_lambda,pac_bayes_backprop}, a variant of the traditional PAC-Bayes bounds~\cite{pacbayes_mcallester,pacbayes_langford,pacbayes_seeger,pacbayes_maurer}, which circumvents the cumbersome square root or other nonlinear transform of the $\textrm{KL}$ term. 
The second theorem is based on the recent regression analysis technique~\cite{pati18,bai20}.  
Without loss of generality we assume $|D_i|\!=\!n$ for all episodes $i$. We let $(q^*(\phi),\{q_i^*(\theta_i)\}_{i=1}^\infty)$ be the optimal solution of (\ref{eq:elbo_optim_orig}). 
We leave the proofs for the two theorems in Supplement. 
%
\begin{theorem}[PAC-Bayes-$\lambda$ bound] Let $R_i(\theta)$ be the generalization error of model $\theta$ for the task $i$, more specifically, $R_i(\theta) = \mathbb{E}_{(x,y)\sim \mathcal{T}_i}[-\log p(y|x,\theta)]$. The following holds with probability $1\!-\!\delta$ for arbitrary small $\delta>0$:
\begin{align}
\mathbb{E}_{i\sim \mathcal{T}} \mathbb{E}_{q_i^*(\theta_i)}[R_i(\theta_i)] \ \leq \ \frac{2\epsilon^*}{n},
\end{align}
where $\epsilon^*$ is the optimal value of (\ref{eq:elbo_optim_orig}). 
\label{thm:gen_pac_bayes}
\end{theorem}
%
\begin{theorem}[Bound derived from regression analysis] 
Let $d_H^2(P_{\theta_i},P^i)$ be the expected squared Hellinger distance between the true  distribution $P^i(y|x)$ and model's $P_{\theta_i}(y|x)$ for task $i$. Then the following holds with high probability:
\begin{align}
\mathbb{E}_{i\sim \mathcal{T}} \mathbb{E}_{q_i^*(\theta_i)}[d_H^2(P_{\theta_i}, P^i)] \leq O\Big( \frac{1}{n}\!+\!\epsilon_n^2\!+\!r_n \Big) + \lambda^*, \label{eq:gen_bound}
\end{align}
where 
$\lambda^*\!=\!\mathbb{E}_{i\sim \mathcal{T}} [\lambda_i^*]$, $\lambda_i^*\!=\! 
\min_{\theta\in\Theta} ||\mathbb{E}_{\theta}[y|\cdot]-\mathbb{E}^i[y|\cdot]||_\infty^2$
is the lowest possible regression error 
within $\Theta$, 
and $r_n, \epsilon_n$ are decreasing sequences vanishing to $0$ as $n$ increases. 
\label{thm:gen_regr_anal}
\end{theorem}



\keypoint{Computational complexity.} 
Although we have introduced a principled Bayesian model/framework for FSL with solid theoretical support, the extra steps introduced in our training/test algorithms appear to be more complicated than simple feed-forward workflows. 
To this end, we have analyzed the time complexity of the proposed algorithm 
contrasted with ProtoNet~\cite{protonet}. For fair comparison, our approach adopts the same NCC head on top of the feature space as ProtoNet. Please find the details in the Supplement Material. Despite seemingly increased complexity in the training/test algorithms, our method incurs only constant-factor overhead compared to the minimal-cost ProtoNet. 


\section{Related Work}\label{sec:related}


Due to the limited space it is overwhelming to review all general FSL and meta learning algorithms here. We refer the readers to the excellent comprehensive surveys~\cite{survey1,survey2} on the latest techniques. 
We rather focus on discussing recent Bayesian approaches and relation to ours. 
Although several Bayesian FSL approaches have been proposed before, most of them dealt with only a small fraction of the  network weights (e.g., a readout head alone) as random variables~\cite{cnp,anp,cnaps,gordon2019metaPred,gpdkt,metaqda}. This  considerably limits the benefits from uncertainty modeling of full network parameters. 

Bayesian approaches to MAML~\cite{platipus,bmaml,amortized_bmaml,nguyen2020uncertaintyMAML} are popular probabilistic extensions of the gradient-based adaptation in MAML~\cite{maml} with known theoretical support~\cite{maml_vs_bmaml}. But we find that they are weak in several aspects to be considered as principled Bayesian methods. For instance, Probabilistic MAML (PMAML or PLATIPUS)~\cite{platipus,grant2018bayesMAML} has a similar hierarchical graphical model structure as ours, but their learning algorithm is considerably deviated from the original variational inference objective. Unlike the  original derivation of the KL term measuring the divergence between the posterior and prior on the task-specific variable $\theta_i$, namely $\mathbb{E}_{q(\phi)}[\textrm{KL}(q_i(\theta_i|\phi) || p(\theta_i|\phi))]$ as in (\ref{eq:elbo_finite}), in PMAML they measure the divergence on the global variable $\phi$, aiming to align the two adapted models, one from the support data only $q(\phi|S_i)$ and the other from both support and query $q(\phi|S_i,Q_i)$. VAMPIRE~\cite{nguyen2020uncertaintyMAML} incorporates uncertainty modeling to MAML by extending MAML's point estimate to a distributional one that is learned by variational inference. However, it inherits all computational overheads from MAML, hindering scalability. 
The BMAML~\cite{bmaml} is not a hierarchical Bayesian model, but aims to replace MAML's gradient-based {\em deterministic} adaptation steps by the {\em stochastic} counterpart using the samples (called particles) from 
$p(\theta_i|S_i)$, 
thus adopting stochastic ensemble-based adaptation steps. If we use a single particle instead, it reduces exactly to MAML. Thus existing Bayesian approaches are not directly related to our hierarchical Bayesian perspective.

\section{Evaluation}\label{sec:expmt}

We perform empirical study to demonstrate the superior performance of the proposed Bayesian few-shot learning algorithm dubbed \textbf{NIW-Meta} to the state-of-the-arts.

\subsection{Few-shot Classification}\label{sec:classification}

\keypoint{Standard benchmarks with ResNet backbones.
}
For standard benchmark comparison using the popular ResNet backbones, in particular ResNet-18~\cite{resnet} and WideResNet~\cite{wrn}, we test our method on: {\em mini}Imagenet (Table~\ref{tab:mini}) and {\em tiered}ImageNet (Table~\ref{tab:tiered}). 
We follow the standard protocols (details of experimental settings in Supplement). 
Our NIW-Meta exhibits consistent improvement over the SOTAs for different settings in support set size and backbones.

\begin{table}[t!]
\centering
\begin{footnotesize}
\centering
\begin{tabular}{cccc}
\toprule
Model & Backbone & 1-Shot & 5-Shot \\
\hline
MAML~\cite{maml}\Tstrut & Conv-4 & $48.70 \pm 1.84$ & $63.11 \pm 0.92$ \\
MetaQDA~\cite{metaqda}\Tstrut & Conv-4 & $56.41 \pm 0.80$ & $72.64 \pm 0.62$ \\
\textbf{NIW-Meta (Ours)}\Tstrut & Conv-4 & $\pmb{56.84 \pm 0.76}$ & $\pmb{72.93 \pm 0.53}$ \\
\hline
%
ProtoNet~\cite{protonet}\Tstrut & ResNet-18 & $54.16 \pm 0.82$ & $73.68 \pm 0.65$ \\
AM3~\cite{am3}\Tstrut & ResNet-12 & $65.21 \pm 0.49$ & $75.20 \pm 0.36$ \\
R2D2~\cite{r2d2}\Tstrut & ResNet-12 & $59.38 \pm 0.31$ & $78.15 \pm 0.24$ \\
RelationNet2~\cite{relationnet2}\Tstrut & ResNet-12 & $63.92 \pm 0.98$ & $77.15 \pm 0.59$ \\
MetaOpt~\cite{metaopt}\Tstrut & ResNet-12 & $64.09 \pm 0.62$ & $80.00 \pm 0.45$ \\
SimpleShot~\cite{simpleshot}\Tstrut & ResNet-18 & $62.85 \pm 0.20$ & $80.02 \pm 0.14$ \\
S2M2~\cite{s2m2}\Tstrut & ResNet-18 & $64.06 \pm 0.18$ & $80.58 \pm 0.12$ \\
MetaQDA~\cite{metaqda}\Tstrut & ResNet-18 & $65.12 \pm 0.66$ & $80.98 \pm 0.75$ \\
\textbf{NIW-Meta (Ours)}\Tstrut & ResNet-18 & $\pmb{65.49 \pm 0.56}$ & $\pmb{81.71 \pm 0.17}$ \\
\hline
SimpleShot~\cite{simpleshot}\Tstrut & WRN-28-10 & $63.50 \pm 0.20$ & $80.33 \pm 0.14$ \\
S2M2~\cite{s2m2}\Tstrut & WRN-28-10 & $64.93 \pm 0.18$ & $83.18 \pm 0.22$ \\
MetaQDA~\cite{metaqda}\Tstrut & WRN-28-10 & $67.83 \pm 0.64$ & $84.28 \pm 0.69$ \\
\textbf{NIW-Meta (Ours)}\Tstrut & WRN-28-10 & $\pmb{68.54 \pm 0.26}$ & $\pmb{84.81 \pm 0.28}$ \\
\bottomrule
\end{tabular}
\end{footnotesize}
\caption{
Results with standard backbones 
on \textbf{{\em mini}ImageNet}.
}
\label{tab:mini}
\vspace{-1.0em}
\end{table}

\begin{table}[t!]
\centering
\begin{footnotesize}
\centering
\begin{tabular}{cccc}
\toprule
Model & Backbone & 1-Shot & 5-Shot \\
\hline
MAML~\cite{maml}\Tstrut & Conv-4 & $51.67 \pm 1.81$ & $70.30 \pm 1.75$ \\
ProtoNet~\cite{protonet}\Tstrut & Conv-4 & $53.31 \pm 0.89$ & $72.69 \pm 0.74$ \\
RelationNet2~\cite{relationnet2}\Tstrut & Conv-4 & $\pmb{60.58 \pm 0.72}$ & $72.42 \pm 0.69$ \\
%
%
MetaQDA~\cite{metaqda}\Tstrut & Conv-4 & $58.11 \pm 0.48$ & $74.28 \pm 0.73$ \\
\textbf{NIW-Meta (Ours)}\Tstrut & Conv-4 & $58.82 \pm 0.91$ & $\pmb{74.86 \pm 0.70}$ \\
\hline
TapNet~\cite{tapnet}\Tstrut & ResNet-12 & $63.08 \pm 0.15$ & $80.26 \pm 0.12$ \\
RelationNet2~\cite{relationnet2}\Tstrut & ResNet-12 & $68.58 \pm 0.63$ & $80.65 \pm 0.91$ \\
MetaOpt~\cite{metaopt}\Tstrut & ResNet-12 & $65.81 \pm 0.74$ & $81.75 \pm 0.53$ \\
SimpleShot~\cite{simpleshot}\Tstrut & ResNet-18 & $69.09 \pm 0.22$ & $84.58 \pm 0.16$ \\
MetaQDA~\cite{metaqda}\Tstrut & ResNet-18 & $69.97 \pm 0.52$ & $85.51 \pm 0.58$ \\
\textbf{NIW-Meta (Ours)}\Tstrut & ResNet-18 & $\pmb{70.52 \pm 0.19}$ & $\pmb{85.83 \pm 0.17}$ \\
\hline
LEO~\cite{leo}\Tstrut & WRN-28-10 & $66.33 \pm 0.05$ & $81.44 \pm 0.09$ \\
SimpleShot~\cite{simpleshot}\Tstrut & WRN-28-10 & $69.75 \pm 0.20$ & $85.31 \pm 0.15$ \\
S2M2~\cite{s2m2}\Tstrut & WRN-28-10 & $73.71 \pm 0.22$ & $88.59 \pm 0.14$ \\
MetaQDA~\cite{metaqda}\Tstrut & WRN-28-10 & $74.33 \pm 0.65$ & $89.56 \pm 0.79$ \\
\textbf{NIW-Meta (Ours)}\Tstrut & WRN-28-10 & $\pmb{74.59 \pm 0.33}$ & $\pmb{89.76 \pm 0.23}$ \\
\bottomrule
\end{tabular}
\end{footnotesize}
\caption{
Results with standard backbones 
on \textbf{{\em tiered}ImageNet}.
}
\label{tab:tiered}
\end{table}

\keypoint{Large-scale ViT backbones.}
We also test our method on the large-scale (pretrained) ViT backbones DINO-small (Dino/s) and DINO-base (DINO/b)~\cite{caron2021emerging}, similarly as the setup in~\cite{hu2022pmf}. 
We summarize in Table~\ref{tab:vit} the results on the three benchmarks:  {\em mini}Imagenet, CIFAR-FS, and {\em tiered}ImageNet. Our NIW-Meta adopts the same NCC head as ProtoNet after the ViT feature extractor. As claimed in~\cite{hu2022pmf}, using the pretrained feature extractor and further finetuning it significantly boost the performance of few-shot learning algorithms including ours. Among the competing methods, our approach yields the highest accuracy for most cases. In particular, compared to the shallow Bayesian MetaQDA~\cite{metaqda}, treating all network weights as random variates in our model turns out to be more effective than the readout parameters alone.

\begin{table}[t!]
\setlength{\tabcolsep}{2.42pt}
\centering
\begin{footnotesize}
\centering
\begin{tabular}{cccccccc}
\toprule
\multirow{2}{*}{Model} & \scriptsize{Backbone} & \multicolumn{2}{c}{\scriptsize{{\em mini}ImageNet}} & \multicolumn{2}{c}{\scriptsize{CIFAR-FS}} & \multicolumn{2}{c}{\scriptsize{{\em tiered}ImageNet}} \\
\cline{3-4} \cline{5-6} \cline{7-8}
& \scriptsize{/ Pretrain} & \scriptsize{1-shot} & \scriptsize{5-shot} & \scriptsize{1-shot} & \scriptsize{5-shot} & \scriptsize{1-shot} & \scriptsize{5-shot} \\
\hline
ProtoNet~\cite{protonet}\Tstrut & DINO/s & $93.1$ & $98.0$ & $81.1$ & $92.5$ & 
$89.0$ & $95.8$ \\
MetaOpt~\cite{metaopt}\Tstrut & DINO/s & $92.2$ & $97.8$ & $70.2$ & $84.1$ & 
$87.5$ & $94.7$ \\
MetaQDA~\cite{metaqda}\Tstrut & DINO/s & $92.0$ & $97.0$ & $77.2$ & $90.1$ & 
$87.8$ & $95.6$ \\
%
\textbf{NIW-Meta (Ours)}\Tstrut & DINO/s & $\pmb{93.4}$ & $\pmb{98.2}$ & $\pmb{82.8}$ & $\pmb{92.9}$ & $\pmb{89.3} $ & $\pmb{96.0}$ \\
\hline
ProtoNet~\cite{protonet}\Tstrut & DINO/b & $95.3$ & $98.4$ & $84.3$ & $92.2$ & 
$91.2$ & $96.5$ \\
MetaOpt~\cite{metaopt}\Tstrut & DINO/b & $94.4$ & $98.4$ & $72.0$ & $86.2$ & 
$89.5$ & $95.7$ \\
MetaQDA~\cite{metaqda}\Tstrut & DINO/b & $94.7$ & $\pmb{98.7}$ & $80.9$ & $\pmb{93.8}$ & 
$89.7$ & $96.5$ \\
%
\textbf{NIW-Meta (Ours)}\Tstrut & DINO/b & $\pmb{95.5} $ & $\pmb{98.7}$ & $\pmb{84.7} $ & $93.2$ & $\pmb{91.4} $ & $\pmb{96.7}$ \\
\bottomrule
\end{tabular}
\end{footnotesize}
\vspace{+0.1em}
\caption{
Classification results with large-scale 
ViT backbones.
}
\label{tab:vit}
\end{table}

\keypoint{Set-based adaptation backbones.} 
We also conduct experiments using the set-based adaptation architecture called FEAT introduced in~\cite{feat}. The network is tailored for few-shot adaptation, namely $y^Q = G(x^Q,S;\theta)$ where the network $G$ takes the entire support set $S$ and query image $x^Q$ as input. 
Note that our NIW-Meta can incorporate any network architecture, even the set-based one like FEAT. As shown in Table~\ref{tab:feat}, the  Bayesian treatment leads to further improvement over~\cite{feat} with this set-based architecture.

\begin{table}[t!]
\setlength{\tabcolsep}{2pt}
\centering
\begin{footnotesize}
\centering
\begin{tabular}{ccccc}
\toprule
\multirow{2}{*}{Model} & \multicolumn{2}{c}{{\em mini}ImageNet} & \multicolumn{2}{c}{{\em tiered}ImageNet} \\
\cline{2-3} \cline{4-5}
& 1-shot & 5-shot & 1-shot & 5-shot \\
\hline
FEAT~\cite{feat}\Tstrut & $66.78$ & $82.05$ & $70.80^{\pm 0.23}$ & $84.79^{\pm 0.16}$ \\
\textbf{NIW-Meta (Ours)}\Tstrut & $\pmb{66.91^{\pm 0.10}}$ & $\pmb{82.28^{\pm 0.15}}$ & $\pmb{70.93^{\pm 0.27}}$ & $\pmb{85.20^{\pm 0.19}}$ \\
\bottomrule
\end{tabular}
\end{footnotesize}
\caption{Comparison between FEAT~\cite{feat} and our method equipped with the same set-based architecture as FEAT. 
}
\label{tab:feat}
\vspace{-1.0em}
\end{table}

\keypoint{Error calibration.} 
One of the key merits of Bayesian modeling is that we have a better calibrated model than deterministic counterparts. We measure the {\em expected calibration errors} (ECE)~\cite{ece} to judge how well the prediction accuracy and the prediction confidence are aligned. More specifically, $ECE = \sum_{b=1}^B \frac{N_b}{N} |acc(b)-conf(b)|$ where we partition test instances into $B$ bins along the model's prediction confidence scores, and $conf(b)$, $acc(b)$ are the average confidence and accuracy for the $b$-th bin, respectively. 
The results on {\em mini}ImageNet with Conv-4 and WRN backbones are shown in Table~\ref{tab:ece_mini}. We used 20 bins and optionally performed the softmax temperature search on validation sets, 
similarly as~\cite{metaqda}. 
Again, Bayesian inference of whole network weights in our NIW-Meta leads to a far better calibrated model than the shallow counterpart Meta-QDA~\cite{metaqda}.

\begin{table}[t!]
\setlength{\tabcolsep}{4pt}
\centering
\begin{footnotesize}
\centering
\begin{tabular}{cccccc}
\toprule
\multirow{2}{*}{Model} & \multirow{2}{*}{Backbone} & \multicolumn{2}{c}{ECE} & \multicolumn{2}{c}{ECE$+$TS} \\
\cline{3-4} \cline{5-6}
& & 1-shot & 5-shot & 1-shot & 5-shot \\
\hline
Linear classifier\Tstrut & Conv-4 & $8.54$ & $7.48$ & $3.56$ & $2.88$ \\
SimpleShot~\cite{simpleshot}\Tstrut & Conv-4 & $33.45$ & $45.81$ & $3.82$ & $3.35$ \\
MetaQDA-MAP~\cite{metaqda}\Tstrut & Conv-4 & $8.03$ & $5.27$ & $2.75$ & $0.89$ \\
MetaQDA-FB~\cite{metaqda}\Tstrut & Conv-4 & $4.32$ & $2.92$ & $2.33$ & $0.45$ \\
\textbf{NIW-Meta (Ours)}\Tstrut & Conv-4 & $\pmb{2.68}$ & $\pmb{1.88}$ & $\pmb{1.47}$ & $\pmb{0.32}$ \\
\hline
SimpleShot~\cite{simpleshot}\Tstrut & WRN-28-10 & $39.56$ & $55.68$ & $4.05$ & $1.80$ \\
S2M2$+$Linear~\cite{s2m2}\Tstrut & WRN-28-10 & $33.23$ & $36.84$ & $4.93$ & $2.31$ \\
MetaQDA-MAP~\cite{metaqda}\Tstrut & WRN-28-10 & $31.17$ & $17.37$ & $3.94$ & $0.94$ \\
MetaQDA-FB~\cite{metaqda}\Tstrut & WRN-28-10 & $30.68$ & $15.86$ & $2.71$ & $0.74$ \\
\textbf{NIW-Meta (Ours)}\Tstrut & WRN-28-10 & $\pmb{10.79}$ & $\pmb{7.11}$ & $\pmb{2.03}$ & $\pmb{0.65}$ \\
\bottomrule
\end{tabular}
\end{footnotesize}
\caption{Expected calibration errors (ECE)  on {\em mini}ImageNet. ``ECE$+$TS'' indicates extra tuning of the temperature hyperparameter (default $=1.0$) in the logit-softmax transformation. 
}
\label{tab:ece_mini}
\vspace{-1.0em}
\end{table}



\subsection{Few-shot Regression}\label{sec:regression}

\keypoint{
Sine-Line dataset~\cite{platipus}.} 
It 
consists of $1D$ $(x,y)$ pairs randomly generated by either linear or sine curves with different scales/slopes/frequencies/phases. For the episodic few-shot learning setup, we follow the standard protocol: each episode is comprised of $k=5$-shot support and 
45 query samples randomly drawn from a random curve (regarded as a task). 
To deal with real-valued targets, we 
adopt the so-called \textbf{RidgeNet}, which has a parameter-free readout head derived from the support data via (closed-form) estimation of the linear coefficient matrix using the ridge regression. It is analogous to the ProtoNet~\cite{protonet} in classification which has a parameter-free head derived from NCC 
on support data. A similar model was introduced in~\cite{r2d2} but mainly repurposed for classification. 
We find that RidgeNet leads to much more accurate prediction than the conventional trainable linear head. 
For instance, the test errors 
are: RidgeNet $=0.82$ vs.~MAML with linear head $=1.86$. Furthermore, we adopt the ridge head in other models as well, such as MAML, 
PMAML~\cite{platipus}, 
and our NIW-Meta. 
See Table~\ref{tab:sineline} for the mean squared errors contrasting our NIW-Meta against competing meta learning methods. The table also contains the regression-ECE (R-ECE) calibration errors\footnote{The definition of the R-ECE is quite different from that of the classification ECE in Sec.~\ref{sec:classification}. We follow the notion of {\em goodness of cumulative distribution matching} used in~\cite{r-ece_mlst,r-ece_nips}. Specifically,  denoting by $\hat{Q}_p(x)$ the $p$-th quantile of the predicted distribution $\hat{p}(y|x)$, we measure the deviation of $p_{true}(y \leq \hat{Q}_p(x)|x)$ from $p$ by absolute difference. So it is $0$ for the ideal case $\hat{p}(y|x)=p_{true}(y|x)$. We use empirical CDF estimates and equal-size binning (20 bins) for $p\in[0,1]$ values. Note that by definition we can only measure R-ECE for models with {\em probabilistic} output $\hat{p}(y|x)$. 
} for the Bayesian methods, PMAML~\cite{platipus} and ours, which clearly shows that our model is better calibrated.

\begin{table}[t!]
\centering
\begin{footnotesize}
\centering
\begin{tabular}{ccc}
\toprule
Model & Mean squared error & R-ECE \\
\hline
RidgeNet\Tstrut & $0.8210$ & N/A \\
MAML (1-step)~\cite{maml}\Tstrut & $0.8206$ & N/A \\
MAML (5-step)~\cite{maml}\Tstrut & $0.8309$ & N/A \\
PMAML (1-step)~\cite{platipus}\Tstrut & $0.9160$ & $0.2666$ \\
\textbf{NIW-Meta (Ours)}\Tstrut & $\pmb{0.7822}$ & $\pmb{0.1728}$ \\
\bottomrule
\end{tabular}
\end{footnotesize}
\vspace{+0.3em}
\caption{Few-shot regression results on the Sine-Line dataset. All methods here adopt the (parameter-free) ridge regression head with L2 regularization coefficient $\lambda\!=\!0.1$, which is significantly more accurate than conventional linear trainable head. PMAML with 5 inner steps incurred numerical errors. 
}
\label{tab:sineline}
\vspace{-1.0em}
\end{table}

\keypoint{Object pose estimation on ShapeNet datasets.} 
We consider the recent few-shot regression benchmarks~\cite{Gao_2022_CVPR,maml_wo_mem} which introduced four datasets for object pose estimation: {\em Pascal-1D}, {\em ShapeNet-1D}, {\em ShapeNet-2D}, and {\em Distractor}. In all datasets, the main goal is to estimate the pose (positions in pixel and/or rotation angles) of the target object in an image. Each episode is specified by: i) selecting a target object randomly sampled from a pool of objects with different object categories, and ii) rendering the same object in an image with several different random poses (position/rotation) to generate data instances. There are $k$ support samples (input images and target pose labels) and $k_q$ query samples for each episode. For ShapeNet-1D, for instance, $k$ is randomly chosen from $3$ to $15$ while $k_q=15$.

Pascal-1D and ShapeNet-1D are relatively easier datasets than the rest two as we have uniform noise-free backgrounds. On the other hand, to make the few-shot learning problem more challenging, ShapeNet-2D and Distractor datasets further introduce random (real-world) background images and/or so called the {\em distractors} which are objects randomly drawn and rendered that have nothing to do with the target pose to estimate. Except for Pascal-1D, some object categories are dedicated solely for meta testing and not revealed during training, thus yielding two different test scenarios: {\em intra-category} (IC) and {\em cross-category} (CC), in which the test object categories are seen and unseen, respectively. 

In~\cite{Gao_2022_CVPR}, they test different augmentation strategies in their baselines: conventional {\em data augmentation} on input images (denoted by DA),  {\em task augmentation} (TA)~\cite{task_aug} which adds random noise to the target labels to help reducing the memorization issue~\cite{maml_wo_mem}, and {\em domain randomization} (DR)~\cite{dom_rand} which  randomly generates background images during training. Among several possible combinations reported in~\cite{Gao_2022_CVPR}, we follow the strategies that perform the best. 
For the target error metrics (e.g., position Euclidean distances in pixels for Distractor, rotation angle differences for ShapeNet-1D), we follow the metrics used in~\cite{Gao_2022_CVPR}. For instance, the quaternion metric may sound reasonable in ShapeNet-2D due to the non-uniform, non-symmetric structures that reside in the target space (3D rotation angles). 

The results are summarized in Table~\ref{tab:shapenet_easy} (easier datasets; Pascal-1D and ShapeNet-1D) and Table~\ref{tab:shapenet_hard} (harder ones; ShapeNet-2D and Distractor). 
In~\cite{Gao_2022_CVPR}, they have shown that the set-based backbone networks, especially the Conditional Neural Process (CNP)~\cite{cnp} and Attentive Neural Process (ANP)~\cite{anp} outperform the conventional architectures of the conv-net feature extractor with the linear head that are adapted by MAML~\cite{maml} (except for the Pascal-1D case). Motivated by this, we adopt the same set-based CNP/ANP architectures within our NIW-Meta. In addition, we also test the ridge-head model with the conv-net feature extractor (denoted by \textbf{C$+$R}). 
Two additional competing models contrasted here are: the Bayesian context aggregation in CNP (CNP$+$BA)~\cite{baco} and the use of the functional contrastive learning loss as extra regularization (FCL)~\cite{Gao_2022_CVPR}. 

For the easier datasets (Table~\ref{tab:shapenet_easy}), there is a dataset regime where MAML 
clearly outperforms (Pascal-1D) and underperforms (ShapeNet-1D) the CNP/ANP architectures. Very promisingly, our NIW-Meta consistently performs the best for both datasets, regardless of the choice of the architectures: not just CNP/ANP but also conv-net feature extractor $+$ ridge head (C$+$R).
For the harder datasets (Table~\ref{tab:shapenet_hard}) where MAML is not reported due to the known computational issues and poor performance, our NIW-Meta still exhibits the best test 
performance with CNP/ANP architectures. Unfortunately, the conv-net $+$ ridge head (C$+$R) did not work well, and our conjecture is that the presence of heavy noise and distractors in the input data requires more sophisticated modeling of interaction/relation among the input instances, as is mainly aimed (and successfully done) by CNP/ANP.

\begin{table}[t!]
\setlength{\tabcolsep}{4.85pt}
\vspace{-0.8em}
\centering
\begin{footnotesize}
\centering
\begin{tabular}{cccc}
\toprule
\multirow{2}{*}{Model} & \multirow{2}{*}{Pascal-1D} & \multicolumn{2}{c}{ShapeNet-1D} \\
\cline{3-4}
& & Intra-category & Cross-category \\
\hline
MAML\Tstrut & $1.02 \pm 0.06$ & $17.96$ & $18.79$ \\
CNP~\cite{cnp}\Tstrut & $1.98 \pm 0.22$ & $7.66 \pm 0.18$ & $8.66 \pm 0.19$ \\
ANP~\cite{anp}\Tstrut & $1.36 \pm 0.25$ & $5.81 \pm 0.23$ & $6.23 \pm 0.12$ \\
\hline
NIW-Meta w/ C$+$R\Tstrut & $\pmb{0.89 \pm 0.06}$ & $5.62 \pm 0.38$ & $6.57 \pm 0.39$ \\
NIW-Meta w/ CNP\ \Tstrut & $0.94 \pm 0.15$ & $5.74 \pm 0.17$ & $6.91 \pm 0.18$ \\
NIW-Meta w/ ANP\ \Tstrut & $0.95 \pm 0.09$ & $\pmb{5.47 \pm 0.12}$ & $\pmb{6.06 \pm 0.18}$ \\
\bottomrule
\end{tabular}
\end{footnotesize}
\caption{Pose estimation test errors for Pascal-1D and ShapeNet-1D. The mean squared errors in rotation angle differences. 
Our method NIW-Meta is equipped with three different backbones: C$+$R $=$ a Conv-net feature extractor with the Ridge head, CNP, and ANP. 
Augmentation: TA for Pascal-1D and TA$+$DA for ShapeNet-1D.
}
\label{tab:shapenet_easy}
\vspace{-1.0em}
\end{table}

\begin{table}[t!]
\setlength{\tabcolsep}{2pt}
\vspace{-0.8em}
\centering
\begin{footnotesize}
\centering
\begin{tabular}{ccccc}
\toprule
\multirow{2}{*}{Model} & \multicolumn{2}{c}{ShapeNet-2D} & \multicolumn{2}{c}{Distractor} \\
\cline{2-3} \cline{4-5}
& IC & CC & IC & CC \\
\hline
CNP~\cite{cnp}\Tstrut & $14.20^{\pm 0.06}$ & $13.56^{\pm 0.28}$ & $2.45$ & $3.75$ \\
CNP$+$BA~\cite{baco}\Tstrut & $14.16^{\pm 0.08}$ & $13.56^{\pm 0.18}$ & $2.44$ & $3.97$ \\
CNP$+$FCL~\cite{Gao_2022_CVPR}\Tstrut & $-$ & $-$ & $2.00$ & $3.05$ \\
ANP~\cite{anp}\Tstrut & $14.12^{\pm 0.14}$ & $13.59^{\pm 0.10}$ & $2.65$ & $4.08$ \\
ANP$+$FCL~\cite{Gao_2022_CVPR}\Tstrut & $14.01^{\pm 0.09}$ & $13.32^{\pm 0.18}$ & $-$ & $-$ \\
\hline
NIW-Meta w/ C$+$R\Tstrut & $21.25^{\pm 0.76}$ & $20.82^{\pm 0.43}$ & $8.90^{\pm 0.26}$ & $17.31^{\pm 0.38}$ \\
NIW-Meta w/ CNP\ \Tstrut & $13.86^{\pm 0.20}$ & $13.04^{\pm 0.13}$ & $\pmb{1.80^{\pm 0.01}}$ & $\pmb{2.94^{\pm 0.14}}$  \\
NIW-Meta w/ ANP\ \Tstrut & $\pmb{13.74^{\pm 0.30}}$ & $\pmb{12.95^{\pm 0.48}}$ & $3.10^{\pm 0.48}$ & $5.20^{\pm 0.88}$ \\
\bottomrule
\end{tabular}
\end{footnotesize}
\caption{Pose estimation test errors for ShapeNet-2D and Distractor. Quaternion differences $\times 10^{-2}$ (ShapeNet-2D) and pixel errors (Distractor). 
The same interpretation as Table~\ref{tab:shapenet_easy}. 
Augmentation: TA$+$DA$+$DR for ShapeNet-2D and DA for Distractor.
}
\label{tab:shapenet_hard}
\vspace{-0.2em}
\end{table}

\begin{figure}[t!]
\begin{center}
%
\centering
\includegraphics[trim = 1mm 2mm 5mm 13mm, clip, scale=0.275
]{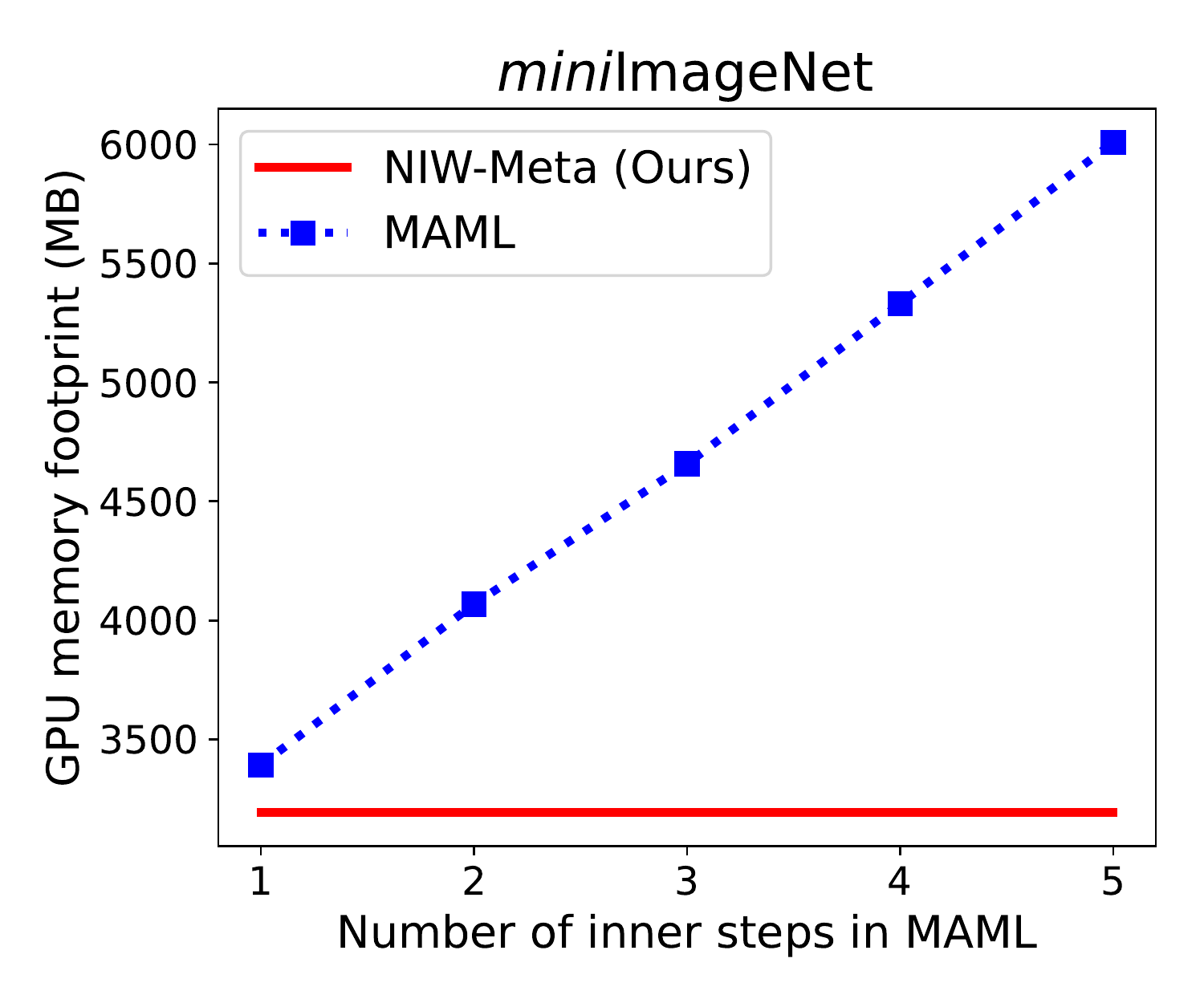}
\ \ 
%
%
\includegraphics[trim = 1mm 2mm 5mm 13mm, clip, scale=0.275
]{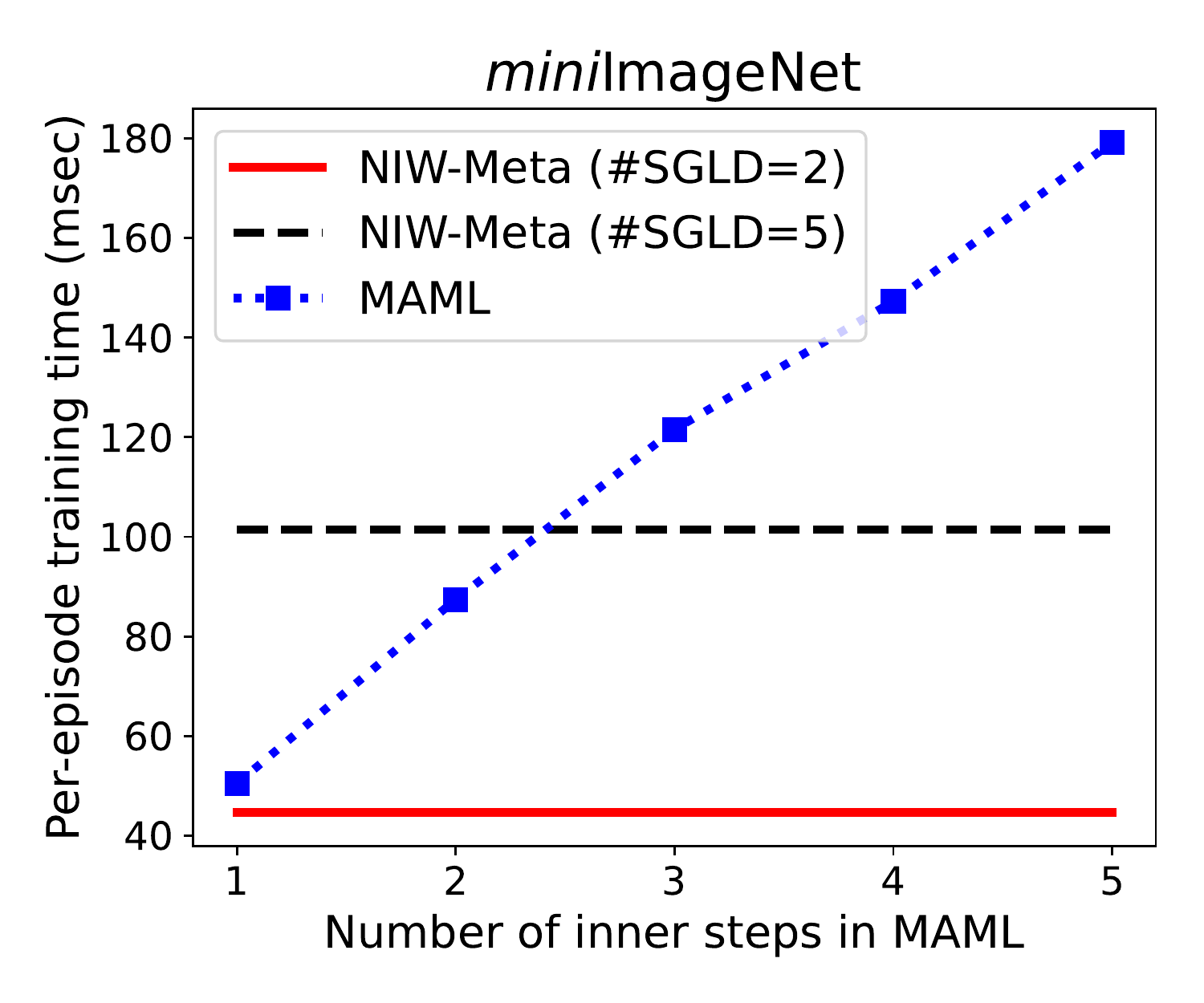}
%
%
\end{center}
\vspace{-1.2em}
\caption{Computational complexity of MAML~\cite{maml} and our NIW-Meta. 
(Left) GPU memory footprints (in MB) for a single batch. (Right) Per-episode training times (in milliseconds). 
}
\vspace{-1.0em}
\label{fig:complexity}
\end{figure}

\subsection{Memory Footprints and Running Times}\label{sec:memory_time}

We claimed in the paper that one of the main drawbacks of MAML~\cite{maml} is the computational overhead to keep track of a large computational graph for inner gradient descent steps. Unlike MAML, our NIW-Meta has a much more efficient episodic optimization strategy, i.e., our one-time optimization only computes the (constant) first/second-order moment statistics of the episodic loss function without storing the full optimization trace. 
To verify this, we measure and compare the memory footprints and running times of MAML and NIW-Meta on two real-world classification/regression datasets: {\em mini}ImageNet 1-shot with the ResNet-18 backbone and ShapeNet-1D with the conv-net backbone. The results in Fig.~\ref{fig:complexity} (ShapeNet-1D in Supp.) show that NIW-Meta has far lower memory requirement than MAML (even smaller than 1-inner-step MAML) while MAML suffers from heavy use of memory space, nearly linearly  increasing as the number of inner steps. The running times of our NIW-Meta are not prohibitively larger compared to MAML where the main computational bottleneck is the SGLD iterations for quadratic approximation of the one-time episodic optimization. We tested two scenarios 
with the number of SGLD iterations 2 and 5, and we have nearly the same (or even better) training speed as the 1-inner-step MAML.

\section{Conclusion}

We have proposed a new hierarchical Bayesian perspective to the episodic FSL problem. By having a higher-level task-agnostic random variate and episode-wise task-specific variables, we formulate a principled Bayesian inference view of the FSL problem with large/infinite evidence. The effectiveness of our approach has been verified empirically in terms of both prediction accuracy and calibration, on a wide range of classification/regression tasks with complex backbones including ViT and set-based adaptation networks.


{\small
\bibliographystyle{ieee_fullname}
\bibliography{main}
}

\onecolumn
\appendix

\centerline{\huge\textbf{{Appendix}}}

\begin{flushleft}
\vspace{+1.0em}
{\large\textbf{{Table of Contents}}}
\end{flushleft}
\begin{itemize}
\vspace{-1.0em}
\item Proofs for Generalization Error Bounds 
(Sec.~\ref{app_sec:generalization})
  \begin{itemize}
      \item Proof for PAC-Bayes-$\lambda$ Bound (Sec.~\ref{app_sec:pac_bayes})
      \item Proof for Regression Analysis Bound (Sec.~\ref{app_sec:regression_analysis})
  \end{itemize}
\item Detailed Derivations (Sec.~\ref{app_sec:derivations})
\item Implementation Details and Experimental Settings
(Sec.~\ref{app_sec:expmt_details})
  \begin{itemize}
      \item Computational Complexity (Sec.~\ref{app_sec:complexity})
  \end{itemize}
\vspace{+2.0em}
\end{itemize}

\section{Proofs for Generalization Error Bounds}\label{app_sec:generalization}

We prove the two theorems Theorem~4.1 and Theorem~4.2 in the main paper that upper-bound the generalization error of the model that is averaged over the learned posterior 
$q(\phi,\theta_{1:\infty})$. 
Without loss of generality we assume $|D_i|\!=\!n$ for all episodes $i$. We let $(q^*(\phi),\{q_i^*(\theta_i)\}_{i=1}^\infty)$ be the optimal solution of Eq.~(9).

\subsection{Proof for PAC-Bayes-$\lambda$ Bound} \label{app_sec:pac_bayes}
      
First, Theorem~4.1, reiterated below as Theorem~\ref{app_thm:gen_pac_bayes}, relates the generalization error to the ultimate ELBO loss Eq.~(9) that we minimized in our algorithm. 

\begin{theorem}[PAC-Bayes-$\lambda$ bound] Let $R_i(\theta)$ be the generalization error of model $\theta$ for the task $i$, more specifically, $R_i(\theta) = \mathbb{E}_{(x,y)\sim \mathcal{T}_i}[-\log p(y|x,\theta)]$. The following holds with probability $1-\delta$ for arbitrary small $\delta>0$:
\begin{align}
\mathbb{E}_{i\sim \mathcal{T}} \mathbb{E}_{q_i^*(\theta_i)}[R_i(\theta_i)] \ \leq \ \frac{2\epsilon^*}{n},
\end{align}
where $\epsilon^*$ is the optimal value of Eq.~(9). 
\label{app_thm:gen_pac_bayes}
\end{theorem}
\begin{proof}
We utilize the recent PAC-Bayes-$\lambda$ bound~\cite{pac_bayes_lambda,pac_bayes_backprop}, a variant of the traditional PAC-Bayes bounds~\cite{pacbayes_mcallester,pacbayes_langford,pacbayes_seeger,pacbayes_maurer}. It states that 
for any $\lambda\in(0,2)$, the following holds with probability at least $1\!-\!\delta$:
\begin{align}
\mathbb{E}_{q(\beta)}[R(\beta)] \leq \frac{1}{1-\lambda/2}\mathbb{E}_{q(\beta)}[\hat{R}_m(\beta)] + \frac{1}{\lambda(1\!-\!\lambda/2)} \frac{\textrm{KL}(q(\beta)||p(\beta)) + \log(2\sqrt{m}/\delta) }{m},
\label{app_eq:pac_bayes_lambda}
\end{align}
where $\beta$ represents all model parameters (random variables), $R(\beta)$ is the generalisation error/loss for a given model $\beta$, and $\hat{R}_m(\beta)$ is the empirical error/loss on the training data of size $m$. It holds for any data-{\em independent} (e.g., prior) distribution $p(\beta)$ and any distribution (possibly data-dependent, e.g., posterior) $q(\beta)$. 

Now we rewrite Eq.~(9) 
in an equivalent form as follows: 
\begin{align}
&\min_{L_0, \{L_i\}_{i=1}^\infty} \ Q(L_0, \{L_i\}_{i=1}^\infty) \ \ \ \ \textrm{where} \\ 
& \ \ \ \ \ \ \ \ Q(L_0, \{L_i\}_{i=1}^\infty) = \frac{1}{N} \bigg(  \mathbb{E}_{q(\phi;L_0) \prod_i q_i(\theta_i;L_i)}\big[
{\textstyle\sum}_i l_i(\theta_i) 
\big] + \textrm{KL}\Big( q(\phi;L_0) {\textstyle\prod}_i q_i(\theta_i;L_i) \ \big|\big| \ p(\phi) {\textstyle\prod}_i p(\theta_i|\phi) \Big)
\bigg)\Bigg|_{N\to\infty} 
\label{app_eq:elbo_optim_orig}
\end{align}
Then we set $\beta := \{\phi,\theta_{1:N}\}$, $q(\beta) := q(\phi) \prod_i q_i(\theta_i)$, and $p(\beta) := p(\phi) {\textstyle\prod}_i p(\theta_i|\phi)$. We also define the generalization loss and the empirical loss as follows:
\begin{gather}
R(\beta) := \frac{1}{N} \sum_{i=1}^N \mathbb{E}_{(x,y)\sim \mathcal{T}_i}[-\log p(y|x,\theta)] = 
\frac{1}{N} \sum_{i=1}^N R_i(\theta)
\\
\hat{R}_m(\beta) := \frac{1}{N} \sum_{i=1}^N \mathbb{E}_{(x,y)\sim D_i}[-\log p(y|x,\theta)] = 
\frac{1}{n}\frac{1}{N} \sum_{i=1}^N -\log p(D_i|\theta_i) = 
\frac{1}{n} \frac{1}{N} \sum_{i=1}^N l_i(\theta_i) 
\end{gather}
Note that the empirical data size $m=n N$ in our case. 
Plugging these into (\ref{app_eq:pac_bayes_lambda}) with $\lambda\!=\!1$ leads to:
\begin{align}
\frac{1}{N} \sum_{i=1}^N \mathbb{E}_{q_i(\theta_i)} [R_i(\theta_i)] \ \leq \ 
2 \bigg( 
  \frac{1}{n} \frac{1}{N} {\textstyle\sum}_{i=1}^N\mathbb{E}_{q_i(\theta_i)}[l_i(\theta_i)] + 
  \frac{\textrm{KL}\big( q(\phi) {\textstyle\prod}_i q_i(\theta_i)  \big|\big|  p(\phi) {\textstyle\prod}_i p(\theta_i|\phi) \big) + \log(2\sqrt{n N}/\delta) }{n N}
\bigg)
\label{app_eq:pacbayes}
\end{align}
Taking $N\!\to\!\infty$ in (\ref{app_eq:pacbayes}) makes i) the LHS become $\mathbb{E}_{i\sim \mathcal{T}} \mathbb{E}_{q_i(\theta_i)}[R_i(\theta_i)]$, ii) 
the complexity term $\frac{\log(2\sqrt{n N}/\delta)}{n N}$ in the RHS vanish, and iii) the RHS converge to $\frac{2}{n} Q(L_0, \{L_i\}_{i=1}^\infty)$. That is, 
\begin{align}
\mathbb{E}_{i\sim \mathcal{T}} \mathbb{E}_{q_i(\theta_i)}[R_i(\theta_i)] \ \leq \ 
\frac{2}{n} Q(L_0, \{L_i\}_{i=1}^\infty).
\label{app_eq:pacbayes_ineq}
\end{align}
Since (\ref{app_eq:pacbayes_ineq}) holds for {\em any} $q$, we take the minimizer $q^*$ of Eq.~(9), which completes the proof.
\end{proof}

\subsection{Proof for Regression Analysis Bound}
\label{app_sec:regression_analysis}

Theorem~4.2, reiterated below as Theorem~\ref{app_thm:gen_regr_anal} in a more detailed form, 
is based on the recent regression analysis techniques~\cite{pati18,bai20}. 
Before we prove the theorem, we formally state some core assumptions and notations. Let $P^i(x,y)$ be the {\em true} data distribution for episode/task $i$ where $i=1,\dots,N$ and $N\to\infty$. We consider regression-based data modeling, assuming that the target $y$ is real vector-valued ($y\in\mathbb{R}^{S_y}$). Also it is assumed that there exists a {\em true regression function} $f^i:\mathbb{R}^{S_x} \to \mathbb{R}^{S_y}$ for each $i$, more formally 
$P^i(y|x) = \mathcal{N}(y; f^i(x), \sigma_\epsilon^2 I)$,
where $\sigma_\epsilon^2$ is constant Gaussian output noise variance. 

For easier analysis we assume that the backbone network is an MLP with $L$ width-$M$ hidden layers, and all activation functions $\sigma(\cdot)$ are Lipschitz continuous with 1. 
We consider the bounded parameter space, 
$\theta \in \Theta = \{ \theta\in\mathbb{R}^G: ||\theta||_\infty \leq B\}$, 
where $G=\dim(\theta)$ and $B$ is the maximal norm bound. Then the prediction (regression) function $f_\theta:\mathbb{R}^{S_x} \to \mathbb{R}^{S_y}$ is induced from  $\theta$ as: $P_\theta(y|x) = \mathcal{N}(y; f_\theta(x), \sigma_\epsilon^2 I)$, where the true noise variance is assumed to be known. 
The expressions $\mathbb{E}_{\theta}[\cdot]$ and $\mathbb{E}^i[\cdot]$ refer to the expectations with respect to model's $P_\theta$ and the true $P^i$, respectively.
The generalisation error measure that we consider is the {\em expected squared Hellinger distance} between the true $P^i$ and the model $P_\theta$, more specifically, 
\begin{align}
d^2(P_\theta, P^i) = \mathbb{E}_{x\sim P^i(x)} \big[ H^2(P_\theta(y|x), P^i(y|x)) \big] 
= \mathbb{E}_{x\sim P^i(x)} \Bigg[ 1 - \exp\bigg(-\frac{||f_\theta(x)-f^i(x)||_2^2}{8 \sigma_\epsilon^2}\bigg) \Bigg].
\end{align}
Now we state our theorem. 

\begin{theorem}[Bound derived from regression analysis] 
Let $d^2(P_{\theta_i},P^i)$ be the expected squared Hellinger distance between the true  distribution $P^i(y|x)$ and model's $P_{\theta_i}(y|x)$ for task/episode $i$. Then the following holds with high probability:
\begin{align}
\mathbb{E}_{i\sim \mathcal{T}} \mathbb{E}_{q_i^*(\theta_i)}[d^2(P_{\theta_i}, P^i)] \ \leq \ \frac{C_0}{n} + C_1 \epsilon_n^2 + C_2 (r_n + \lambda^*), \label{eq:gen_bound2}
\end{align}
where $C_\bullet\!>\!0$ are some constant, 
$\lambda^* = \mathbb{E}_{i\sim \mathcal{T}} [\lambda_i^*]$ with $\lambda_i^* = \min_{\theta\in\Theta} \max_{x} ||\mathbb{E}_{\theta}[y|x]-\mathbb{E}^i[y|x]||^2$ is the lowest possible regression error within the underlying network $\Theta$, 
$r_n = \frac{G}{n} \bigg( (L+1) \log M + \log\Big(S_x \sqrt{\frac{n}{G}}\Big) \bigg)$,
and $\epsilon_n = \sqrt{r_n} \log^\delta(n)$ for $\delta>1$ constant. 
\label{app_thm:gen_regr_anal}
\end{theorem}
\begin{proof}
We utilize the Donsker-Varadhan's (DV) theorem~\cite{dv} to relate the variational ELBO objective function to the Hellinger distance. The DV theorem says that the following inequality holds for any distributions $p$, $q$ and any (bounded) function $h(z)$:
\begin{align}
\log \mathbb{E}_{p(z)}[e^{h(z)}] = \max_q \big( \mathbb{E}_{q(z)}[h(z)] - \textrm{KL}(q||p) \big).
\label{eq:dv}
\end{align}
In our case, we define: $p(z) := p(\theta_i|\phi)$, $q(z) := q_i(\theta_i)$, $h(z) := \log \eta_i(\theta_i)$ with
\begin{align}
\eta_i(\theta_i) := \exp\big( \rho(P_{\theta_i}(D_i), P^i(D_i)) + n d^2(P_{\theta_i},P^i) \big)
\end{align}
where $\rho(P_{\theta_i}(D_i), P^i(D_i)) := \log \frac{P_{\theta_i}(D_i)}{P^i(D_i)}$ is the log-ratio. Note that $P(D_i) = P(Y_i|X_i)$. 
Plugging these into (\ref{eq:dv}) leads to the following inequality which holds for any $\phi$:
\begin{align}
n \cdot \mathbb{E}_{q_i(\theta_i)}[d^2(P_{\theta_i},P^i)] \ &\leq \ 
\mathbb{E}_{q_i(\theta_i)}[-\rho(P_{\theta_i}(D_i), P^i(D_i))] + \textrm{KL}(q_i(\theta_i)||p(\theta_i|\phi)) + \log \mathbb{E}_{p(\theta_i|\phi)}[\eta_i(\theta_i)].
\end{align}
We take the expectation with respect to $q(\phi)$, which yields:
\begin{align}
n \cdot \mathbb{E}_{q_i(\theta_i)}[d^2(P_{\theta_i},P^i)] \ &\leq \ 
\mathbb{E}_{q_i(\theta_i)}[-\rho(P_{\theta_i}(D_i), P^i(D_i))] + \mathbb{E}_{q(\phi)}[\textrm{KL}(q_i(\theta_i)||p(\theta_i|\phi))] + \mathbb{E}_{q(\phi)}\big[\log \mathbb{E}_{p(\theta_i|\phi)}[\eta_i(\theta_i)]\big].
\label{app_eq:bound1}
\end{align}
From the regression theorem~\cite{pati18} (Theorem~3.1 therein), it is known that $\mathbb{E}_{s(\theta)}[\eta(\theta)] \leq e^{C n \epsilon_n^2}$ for any distribution $s(\theta)$ with high probability. We apply this result to the last term of (\ref{app_eq:bound1}). Summing it over $i=1,\dots,N$ leads to:
\begin{align}
n\cdot \sum_{i=1}^N \mathbb{E}_{q_i(\theta_i)}[d^2(P_{\theta_i},P^i)] \ \leq \ 
\sum_{i=1}^N \mathbb{E}_{q_i(\theta_i)}[-\rho(P_{\theta_i}(D_i), P^i(D_i))] + \sum_{i=1}^N \mathbb{E}_{q(\phi)}[\textrm{KL}(q_i(\theta_i)||p(\theta_i|\phi))] + N C n \epsilon_n^2.
\end{align}
By dividing both sides by $N$ and sending $N\to\infty$, we have:
\begin{align}
n\cdot \mathbb{E}_{i\sim\mathcal{T}} \mathbb{E}_{q_i(\theta_i)}[d^2(P_{\theta_i},P^i)] \ \leq \ 
\underbrace{
  \mathbb{E}_{i\sim\mathcal{T}} \Big[ \mathbb{E}_{q_i(\theta_i)}[-\rho(P_{\theta_i}(D_i), P^i(D_i))] + \mathbb{E}_{q(\phi)}[\textrm{KL}(q_i(\theta_i)||p(\theta_i|\phi))] \Big]
}_{= \ -\textrm{ELBO}(q) \ + \ \log P^i(D_i)} + C n \epsilon_n^2.
\label{app_eq:bound2}
\end{align}
As indicated, the right hand side is composed of $-\textrm{ELBO}(q)$ (the objective function of Eq.~(9)), the constant $\log P^i(D_i)$, and the complexity term $C n \epsilon_n^2$.

The next step is to plug in the optimal $q^*$ to have a meaningful upper bound. To this end, we introduce/define  $\tilde{q}_i(\theta_i)$ and $\tilde{q}(\phi)$ as follows:
\begin{gather}
\tilde{q}_i(\theta_i) = \mathcal{N}(\theta_i; \theta_i^*, \sigma_n^2 I), \ \ \tilde{q}(\phi) = \arg\min_{q(\phi)} \ \mathbb{E}_{i \sim \mathcal{T}} \mathbb{E}_{q(\phi)}[\textrm{KL}(\tilde{q}_i(\theta_i)||p(\theta_i|\phi))], \ \ \textrm{where} \\ 
\theta_i^* = \arg\min_{\theta\in\Theta} \max_{x\in\mathbb{R}^{S_x}} ||f_\theta(x)-f^i(x)||^2, \ \ \sigma_n^2 = \frac{G}{8n}A, 
\\
A^{-1} = \log(3S_x M) \cdot (2BM)^{2(L+1)} \cdot \bigg( \Big(S_x+1+\frac{1}{BM-1}\Big)^2 + \frac{1}{(2BM)^2-1} + \frac{2}{(2BM-1)^2} \bigg).
\end{gather}
Since $(\{q_i^*(\theta_i)\}_{i=1}^N, q^*(\phi))$ is the minimizer of the negative ELBO Eq.~(9), we clearly have $-\textrm{ELBO}(q^*) \leq -\textrm{ELBO}(\tilde{q})$. 
We plug $q^*$ into (\ref{app_eq:bound2}) and apply this ELBO inequality to have:
\begin{align}
n\cdot \mathbb{E}_{i\sim\mathcal{T}} \mathbb{E}_{q^*_i(\theta_i)}[d^2(P_{\theta_i},P^i)] \ \leq \ 
  \mathbb{E}_{i\sim\mathcal{T}} \mathbb{E}_{\tilde{q}_i(\theta_i)}[-\rho(P_{\theta_i}(D_i), P^i(D_i))] + \mathbb{E}_{i\sim\mathcal{T}} \mathbb{E}_{\tilde{q}(\phi)}[\textrm{KL}(\tilde{q}_i(\theta_i)||p(\theta_i|\phi))] 
  + C n \epsilon_n^2.
\label{app_eq:bound3}
\end{align}
The second term of the right hand side of (\ref{app_eq:bound3}) is constant (independent of $n$) and denoted by $\tilde{C}$. For the first term of the right hand side, we use the following fact from the proof of Lemma~4.1 in~\cite{bai20}, which says that with high probability,
\begin{align}
\mathbb{E}_{\tilde{q}_i(\theta_i)}[-\rho(P_{\theta_i}(D_i), P^i(D_i))] \ \leq \  C'n(r_n+\lambda_i^*),
\label{eq:E_q_tilde_bound}
\end{align}
for some constant $C'>0$. 
Using this bound, (\ref{app_eq:bound3}) can be written as follows:
\begin{align}
n\cdot \mathbb{E}_{i\sim\mathcal{T}} \mathbb{E}_{q^*_i(\theta_i)}[d^2(P_{\theta_i},P^i)] \ \leq \ 
  \tilde{C} \ + \ C'n \Big( r_n + \mathbb{E}_{i\sim\mathcal{T}}[\lambda_i^*] \Big) \ + \ C n \epsilon_n^2. 
\end{align}
The proof completes by dividing both sides by $n$.
\end{proof}

\section{Detailed Derivations}\label{app_sec:derivations}

\subsection{ELBO Derivation for Eq.~(8)}\label{app_sec:elbo_deriv}

We derive the upper bound of the negative marginal log-likelihood for our Bayesian FSL model, that is, deriving Eq.~(8) in the main paper. 
\begin{align}
&\textrm{KL}\big( q(\phi,\theta_{1:N}) \ || \ p(\phi,\theta_{1:N}|D_{1:N}) \big) \ = \ \mathbb{E}_q \Bigg[ \log \frac{q(\phi) \cdot \prod_i q_i(\theta_i) \cdot p(D_{1:N})}{p(\phi) \cdot \prod_i p(\theta_i|\phi) \cdot \prod_i p(D_i|\theta_i)} \Bigg] \\
&\ \ \ \ \ \ \ \ \ \ \ \ = \ \underbrace{\textrm{KL}(q(\phi)||p(\phi)) + \sum_{i=1}^N \Big( \mathbb{E}_{q_i(\theta_i)}[-\log p(D_i|\theta_i)] +  \mathbb{E}_{q(\phi)}\big[\textrm{KL}(q_i(\theta_i) || p(\theta_i|\phi))\big] \Big)}_{=: \mathcal{L}(L)} 
\ + \ \log p(D_{1:N}).
\end{align}
Since KL divergence is non-negative, $-\mathcal{L}(L)$ must be lower bound of the data log-likelihood $\log p(D_{1:N})$, rendering $\mathcal{L}(L)$ an upper bound of $-\log p(D_{1:N})$.

\subsection{Derivation for $\mathbb{E}_{q(\phi)}\big[\textrm{KL}(q_i(\theta_i) || p(\theta_i|\phi))\big]$ in Eq.~(9--10)}\label{app_sec:elbo_deriv2}

We will derive the full closed-form formula for $\mathbb{E}_{q(\phi)}\big[\textrm{KL}(q_i(\theta_i) || p(\theta_i|\phi))\big]$, which not only leads to equivalence between Eq.~(10) and Eq.~(11), but is also used in deriving Eq.~(14). In a nutshell, the formula that we will prove is as follows:
\begin{align}
\mathbb{E}_{q(\phi)}\big[\textrm{KL}(q_i(\theta_i) || p(\theta_i|\phi))\big] = &
\frac{1}{2} \bigg(\!-\!d \log(2e) + \log \frac{|V_0|}{|V_i|}
  - \psi_d\Big(\frac{n_0}{2}\Big) + \frac{d}{l_0} + 
n_0 \big(m_i\!-\!m_0\big)^\top\!V_0^{-1}\!\big(m_i\!-\!m_0\big) + n_0 \textrm{Tr}\big(V_i V_0^{-1}\big)
\bigg),
\label{app_eq:Ekl_final}
\end{align}
where $\psi_d(a) = \sum_{j=1}^d \psi(a+(1-j)/2)$ is the multivariate digamma function, and $\psi(\cdot)$ is the digamma function. 

We begin with the definition of the KL divergence,
\begin{align}
\mathbb{E}_{q(\phi)}\big[\textrm{KL}(q_i(\theta_i) || p(\theta_i|\phi))\big] = 
-\mathbb{H}(q_i(\theta_i)) + \mathbb{E}_{q(\phi)q_i(\theta_i)}[-\log p(\theta_i|\phi)],
\label{app_eq:EKL_first}
\end{align}
where the first term is the negative entropy which admits a closed form due to Gaussian $q_i(\theta_i) = \mathcal{N}(\theta_i; m_i, V_i)$, 
\begin{align}
-\mathbb{H}(q_i(\theta_i)) = -\frac{d}{2} \log (2 \pi e) - \frac{1}{2} \log |V_i|.
\end{align}
Next we expand the second term of (\ref{app_eq:EKL_first}) using $p(\theta_i|\phi) = \mathcal{N}(\theta_i; \mu, \Sigma)$ as follows:
\begin{align}
\mathbb{E}_{q(\phi)q_i(\theta_i)}[-\log p(\theta_i|\phi)] = 
\underbrace{\frac{1}{2} \mathbb{E}_{q(\phi)}
\big[\log |\Sigma|\big]}_{=:T_1} + \underbrace{\frac{1}{2} \mathbb{E}_{q(\phi)q_i(\theta_i)}\big[(\theta_i-\mu)^\top \Sigma^{-1}(\theta_i-\mu)\big]}_{=:T_2} + \frac{d}{2}\log (2\pi).
\end{align}
Using the following facts from~\cite{bishop:2006:PRML,braun08}:
\begin{align}
\mathbb{E}_{\mathcal{IW}(\Sigma; \Psi, \nu)}\log|\Sigma| \ &= \ -d\log 2 + \log |\Psi| - \psi_d(\nu/2)
 \label{app_eq:fact_1} \\
\mathbb{E}_{\mathcal{IW}(\Sigma; \Psi, \nu)}\Sigma^{-1} \ &= \ \nu\Psi^{-1}, \label{app_eq:fact_2}
\end{align}
we can derive the two terms $T_1$ and $T_2$ as follows (Recall: $q(\phi) = \mathcal{N}(\mu; m_0, l_0^{-1}\Sigma) \cdot  \mathcal{IW}(\Sigma; V_0, n_0)$):
\allowdisplaybreaks
\begin{align}
&(T_1=) \ \frac{1}{2} \mathbb{E}_{q(\phi)}
\big[\log |\Sigma|\big] \ = \ 
\frac{1}{2} \bigg( 
  -d \log 2 + \log |V_0|
  - \psi_d\Big(\frac{n_0}{2}\Big) \bigg) \\
&(T_2=) \ \frac{1}{2} \mathbb{E}_{q(\phi)q_i(\theta_i)}\big[(\theta_i-\mu)^\top \Sigma^{-1}(\theta_i-\mu)\big] \ = \ 
  \frac{1}{2} \mathbb{E}_{q(\phi)q_i(\theta_i)}\textrm{Tr}\Big( (\theta_i-\mu) (\theta_i-\mu)^\top \Sigma^{-1}\Big) \\
  & \ \ \ \ \ \ \ \ \ \ \ \ \ \ \ \ \ \ \ \ \ \ = \ \frac{1}{2} \textrm{Tr}\Big( \mathbb{E}_{q(\phi)}\Big[ \mathbb{E}_{q_i(\theta_i)}\big[(\theta_i-\mu) (\theta_i-\mu)^\top\big] \Sigma^{-1} \Big] \Big) \\
  & \ \ \ \ \ \ \ \ \ \ \ \ \ \ \ \ \ \ \ \ \ \ = \ \frac{1}{2} \textrm{Tr}\Big( \mathbb{E}_{q(\phi)}\Big[ \big( m_i m_i^\top - \mu m_i^\top -m_i\mu^\top + \mu \mu^\top + V_i \big) \Sigma^{-1} \Big] \Big) \\
  & \ \ \ \ \ \ \ \ \ \ \ \ \ \ \ \ \ \ \ \ \ \ = \ \frac{1}{2} \textrm{Tr}\Big( \mathbb{E}_{\mathcal{IW}(\Sigma; V_0, n_0)}\Big[ \mathbb{E}_{\mathcal{N}(\mu; m_0, l_0^{-1}\Sigma)} \big[ m_i m_i^\top - \mu m_i^\top -m_i\mu^\top + \mu \mu^\top + V_i \big] \Sigma^{-1} \Big] \Big) \\
  & \ \ \ \ \ \ \ \ \ \ \ \ \ \ \ \ \ \ \ \ \ \ = \ \frac{1}{2} \textrm{Tr}\Big( \mathbb{E}_{\mathcal{IW}(\Sigma; V_0, n_0)}\Big[ \big( m_i m_i^\top - m_0 m_i^\top -m_i m_0^\top + m_0 m_0^\top + l_0^{-1}\Sigma + V_i \big) \Sigma^{-1} \Big] \Big) \\
  & \ \ \ \ \ \ \ \ \ \ \ \ \ \ \ \ \ \ \ \ \ \ = \ \frac{1}{2} \textrm{Tr}\Big( 
  \frac{1}{l_0} I + 
  \big( (m_i-m_0)(m_i-m_0)^\top + V_i \big) n_0 V_0^{-1} \Big) \\ 
  & \ \ \ \ \ \ \ \ \ \ \ \ \ \ \ \ \ \ \ \ \ \ = \ \frac{1}{2} \bigg( 
  \frac{d}{l_0} + n_0 \big(m_i - m_0\big)^\top V_0^{-1} \big(m_i - m_0\big) + n_0 \textrm{Tr}\big(V_i V_0^{-1}\big) \bigg)
\end{align}
Combining all the above results yields the formula (\ref{app_eq:Ekl_final}).

\subsection{Derivation for Eq.~(11) from Eq.~(10)} 

Using the result (\ref{app_eq:Ekl_final}), we can easily show that the one-time episodic optimization Eq.~(10) in the main paper ((\ref{app_eq:ot_optim1}) below) reduces to Eq.~(11) ((\ref{app_eq:ot_optim2}) below). 
\begin{gather}
\min_{L_i} \ \mathbb{E}_{q_i(\theta_i;L_i)}[l_i(\theta_i)] + \mathbb{E}_{q(\phi)} \big[\textrm{KL}(q_i(\theta_i;L_i) || p(\theta_i|\phi))\big]
\label{app_eq:ot_optim1} \\
\min_{m_i,V_i} \ \mathbb{E}_{\mathcal{N}(\theta_i;m_i,V_i)}[l_i(\theta_i)] - \frac{1}{2} \log |V_i| + \frac{n_0}{2} (m_i-m_0)^\top V_0^{-1} (m_i-m_0) + \frac{n_0}{2} \textrm{Tr}\big(V_i V_0^{-1}\big) 
\label{app_eq:ot_optim2}
\end{gather}
Recall that the optimization is with respect to $L_i=(m_i,V_i)$ with $L_0=\{m_0,V_0,l_0,n_0\}$ fixed.
Plugging (\ref{app_eq:Ekl_final}) into (\ref{app_eq:ot_optim1}) and removing the terms other than $(m_i,V_i)$ leads to (\ref{app_eq:ot_optim2}).

\subsection{Derivation for Eq.~(13)}

For the quadratic approximation of $l_i(\theta_i) = -\log p(D_i|\theta_i) \approx \frac{1}{2}(\theta_i\!-\!\overline{m}_i)^\top \overline{A}_i(\theta_i\!-\!\overline{m}_i) + \textrm{const.}$, here we show that the minimizer of Eq.~(11) ((\ref{app_eq:ot_optim2}) above) can be obtained by the closed-form formula Eq.~(13) ((\ref{app_eq:miVi_star}) below).
\begin{align}
m_i^*(L_0) = (\overline{A}_i + n_0 V_0^{-1})^{-1} (\overline{A}_i \overline{m}_i + n_0 V_0^{-1} m_0), \ \ \ \ \ \ \ \ 
V_i^*(L_0) = (\overline{A}_i + n_0 V_0^{-1})^{-1}.
\label{app_eq:miVi_star}
\end{align}
%
By replacing $l_i(\theta_i)$ by the quadratic approximation, the expected loss term in Eq.~(11) or (\ref{app_eq:ot_optim2}) can be written as follows:
\begin{align}
\mathbb{E}_{\mathcal{N}(\theta_i;m_i,V_i)}[l_i(\theta_i)] \ &\approx \
\mathbb{E}_{\mathcal{N}(\theta_i;m_i,V_i)}\Big[\frac{1}{2} (\theta_i-\overline{m}_i)^\top \overline{A}_i(\theta_i-\overline{m}_i) \Big] + \textrm{const.} \\
  &= \ \frac{1}{2} \Big( 
  \textrm{Tr}\big(\mathbb{E}[\theta\theta^\top] \overline{A}_i \big) - \overline{m}_i^\top \overline{A}_i m_i - m_i^\top \overline{A}_i \overline{m}_i + \overline{m}_i^\top \overline{A}_i \overline{m}_i 
  \Big) + \textrm{const.} \\
  &= \ \frac{1}{2} \Big( 
  \textrm{Tr}\big(V_i \overline{A}_i \big) + 
  m_i^\top \overline{A}_i m_i - \overline{m}_i^\top \overline{A}_i m_i - m_i^\top \overline{A}_i \overline{m}_i + \overline{m}_i^\top \overline{A}_i \overline{m}_i 
  \Big) + \textrm{const.} \\
  &= \ \frac{1}{2} \Big( 
  \textrm{Tr}\big(V_i \overline{A}_i \big) + 
  (m_i - \overline{m}_i)^\top \overline{A}_i (m_i - \overline{m}_i) \Big) + \textrm{const.}
\end{align}
After plugging this back to (\ref{app_eq:ot_optim2}), we take the derivatives of the objective with respect to $m_i$ and $V_i$ and set them to $0$:
\begin{gather}
\nabla_{m_i}(\cdot) \ = \ \overline{A}_i (m_i - \overline{m}_i) + n_0 V_0^{-1} (m_i-m_0) \ = \ 0
\\
\nabla_{V_i}(\cdot) \ = \ \frac{1}{2} \Big( 
  \overline{A}_i - V_i^{-1} + n_0 V_0^{-1}
\Big) \ = \ 0
\end{gather}
The solution becomes Eq.~(13) or (\ref{app_eq:miVi_star}).

\subsection{Derivation for Eq.~(14)}

It is quite straightforward that by plugging Eq.~(13) or (\ref{app_eq:miVi_star}) and also (\ref{app_eq:Ekl_final}) in Eq.~(9), we have our final optimization problem Eq.~(14) in the main paper. It is reiterated below:
\begin{align}
&\min_{L_0} \ \mathbb{E}_{i\sim\mathcal{T}} \Big[ f_i(L_0) + \frac{1}{2} g_i(L_0) + \frac{d}{2 l_0} \Big] \ \ \textrm{s.t.} \label{app_eq:ultimate_optim} \\
& \ \ \ \ \ \ f_i(L_0) \ = \ \mathbb{E}_{\epsilon\sim \mathcal{N}(0,I)}\Big[l_i\Big(m_i^*(L_0)+V_i^*(L_0)^{1/2}\epsilon\Big)\Big], \\ 
& \ \ \ \ \ \ g_i(L_0) \ = \ \log \frac{|V_0|}{|V_i^*(L_0)|} + n_0 \textrm{Tr}\big(V_i^*(L_0)V_0^{-1}\big) \ + \ \ n_0 \big(m_i^*(L_0)\!-\!m_0\big)^\top V_0^{-1} \big(m_i^*(L_0)\!-\!m_0\big) - \psi_d\Big(\frac{n_0}{2}\Big),
\end{align}

\subsection{Formulas for Test-Time ELBO Optimization Eq.~(18)}

We provide formulas for the test-time ELBO in Eq.~(18) ((\ref{app_eq:test_elbo}) below). For the test-time variational density $v(\theta) = \mathcal{N}(\theta; m, V)$ to approximate $p(\theta|D^*,\phi^*)$ for test support data $D^*$ and learned $\phi^*=(\mu^*\!=\!m_0,\Sigma^*\!=\!V_0/(n_0\!+\!d\!+\!2))$, we had 
\begin{align}
\min_{m,V} \  \mathbb{E}_{v(\theta)}[-\log p(D^*|\theta)] + \textrm{KL}(v(\theta)||p(\theta|\phi^*)).
\label{app_eq:test_elbo}
\end{align}
Using the closed-form Gaussian KL divergence and the reparametrized sampling trick, we can express (\ref{app_eq:test_elbo}) as:
\begin{align}
\min_{m,V} \  \mathbb{E}_{\epsilon\sim\mathcal{N}(0,I)}\big[-\log p\big(D^*|m+V^{1/2}\epsilon\big)\big] - \frac{1}{2} \log |V|
+ \frac{n_0\!+\!d\!+\!2}{2} \Big(
  \textrm{Tr}\big(V_0^{-1}V\big) + (m-m_0)^\top V_0^{-1} (m-m_0) \Big).
\label{app_eq:test_elbo_formula}
\end{align}
Also, our meta-test prediction algorithm is summarized as a pseudo code in Alg.~\ref{app_alg:test}.

\begin{algorithm}[t!]
  \caption{Meta-test prediction algorithm.}
  \label{app_alg:test}
\begin{small}
\begin{algorithmic}
   \STATE {\bfseries Input:} Test support data $D^*$ and learned $q(\phi;L_0)$ where $L_0=\{m_0,V_0,n_0\}$. \\ \ \ \ \ \ \ \ \ \ \ \ \ $M_V=$ number of test-time variational inference steps. \\
   \ \ \ \ \ \ \ \ \ \ \ \ $M_S=$ number of test-time model samples.
   \STATE Compute the mode $\phi^*=(\mu^*\!=\!m_0,\Sigma^*\!=\!V_0/(n_0\!+\!d\!+\!2))$.
   \STATE Initialize $(m,V)$ with $(\mu^*,\Sigma^*)$.
   \FOR{$i=1,\dots,M_V$} 
      \STATE Take a gradient descent update for $(m,V)$ with the objective in (\ref{app_eq:test_elbo_formula}).
   \ENDFOR
   \STATE Sample $\theta^{(s)} \sim \mathcal{N}(\theta; m, V)$ for $s=1,\dots,M_S$.
   \STATE {\bfseries Output:} Sample-averaged predictive distribution, $p(y^*|x^*,D^*,D_{1:\infty}) \approx 
\frac{1}{S} \sum_{s=1}^{M_S} p(y^*|x^*,\theta^{(s)})$.
\end{algorithmic}
\end{small}
\end{algorithm}


\section{Implementation Details and Experimental Settings}\label{app_sec:expmt_details}

We implement our NIW-Meta using  PyTorch~\cite{paszke2017automatic} and the Higher~\cite{higher}\footnote{\url{https://github.com/facebookresearch/higher}} library. The latter makes the implementation of the backpropagation through the functional network weights in PyTorch modules very easy. 
Real codes for the synthetic SineLine regression dataset and the large-scale ViT are also provided in the Supplement to help understanding of our algorithm.  For all few-shot classification experiments, we use the ProtoNet-like parameter-free NCC head in our NIW-Meta. 
Some important implementation details on the SGLD iterations for quadratic approximation of the one-time episode optimization include: we have either 3 steps without burn-in (for large-scale backbones ViT) or 5 steps with 2 burn-in steps (for smaller backbones ConvNet, ResNet-18, and CNP). Before starting SGLD iterations, the network is initialized with the current model parameters $m_0$. For reliable variance estimation of $\overline{A}_i$, a small regularizer is added to the diagonal entries of the variances.

For the standard benchmarks with ConvNet/ResNet backbones, we follow the standard protocols of~\cite{simpleshot,s2m2,metaqda}: With 64/16/20 and 391/97/160 train/validation/test class splits for {\em mini}ImageNet and {\em tiered}ImageNet datasets, respectively, the images are resized to 84 pixels. 
We initialize the $m_0$ parameters from the pretrained models: checkpoints from~\cite{simpleshot} for Conv-4 and ResNet-18 and checkpoints from~\cite{s2m2} for WRN-28-10. 
With the stochastic gradient descent (SGD) optimizer, we set momentum 0.9, weight decay 0.0001, and initial learning rate 0.01 for {\em mini}ImageNet and 0.001 for {\em tiered}ImageNet. We have learning rate schedule by reducing the learning rate by the factor of 0.1 at epoch 70.

For the large-scale ViT backbones, we utilize the code base from~\cite{hu2022pmf}. We use the self-supervised pretrained checkpoints from~\cite{caron2021emerging} to initialize the $m_0$ parameters. The CIFAR-FS dataset is formed by splitting the original CIFAR-100 into 64/16/20 train/validation/test classes. For training, we run 100 epochs, each epoch comprised of 2000 episodes. We follow the same warm-up plus cosine annealing learning rate
scheduling as~\cite{hu2022pmf}. For test evaluation, we have 600 episodes from the test splits. 

For the few-shot regression experiments with ShapeNet datasets, we basically follow all experimental settings and CNP/ANP network architectures from~\cite{Gao_2022_CVPR}. For instance, in the ShapeNet-1D dataset, we run our algorithm for $500K$ iterations with learning rate $10^{-4}$ where each batch iteration consists of 10 episodes. The CNP backbone, for instance, in the Distractor dataset case, has a ResNet image encoder and a linear target encoder, where the concatenated instance-wise embeddings then go through a three-layer fully connected network followed by max pooling. The decoder has a similar architecture and converts the support set embedding and a query image into a target label. 
For the conv-net plus ridge-regression head backbone (C$+$R) tested for our method, the conv-net feature extractors are formed by taking the encoder parts of the CNP architectures in~\cite{Gao_2022_CVPR} while discarding the pooling operations and decoders. Also the ridge-regression L2 regularization coefficient is set to $\lambda=1.0$ for all datasets.

\subsection{Computational Complexity}\label{app_sec:complexity}

\begin{table}
\centering
\begin{small}
\centering
\begin{tabular}{ccc}
\toprule
& Training time & Test time \\
\hline
\multirow{2}{*}{NIW-Meta}\Tstrut & $(F_S\!+\!F_Q\!+\!B_Q)\cdot (M_{L}\!+\!1)$ & $(F_S\!+\!B_S)\cdot M_{V} \ +$ \\ 
& $+\ O(d)$ & $(F_S\!+\!F_Q)\cdot M_{S} + O(d)$ \\ 
\hline
ProtoNet\Tstrut & $F_S\!+\!F_Q\!+\!B_Q$ & $F_S\!+\!F_Q$ \\ 
\bottomrule
\end{tabular}
\end{small}
\vspace{+0.5em}
\caption{(Per-episode) Time complexity of our NIW-Meta vs.~ProtoNet. 
We denote by $F_D$ and $B_D$ the forward-pass and backpropagation times with data $D=S$upport or $Q$uery. In our algorithm, $M_{L}$, $M_{V}$, and $M_{S}$ indicate the numbers of SGLD iterations, test-time variational inference steps for Eq.~(18) or (\ref{app_eq:test_elbo},\ref{app_eq:test_elbo_formula}), and test-time model samples $\theta^{(s)}$, respectively. The costs required for reparametrized sampling in model space and regularizer computation in Eq.~(14) or (\ref{app_eq:ultimate_optim}) are denoted by $O(d)$ where $d=$ number of backbone parameters. 
}
\label{app_tab:complexity}
\end{table}

In this section we analyze the computational complexity of the proposed algorithm NIW-Meta. 
First, we analyze the time complexity and contrast it with that of ProtoNet~\cite{protonet}. For fair comparison, our approach adopts the same NCC head on top of the feature space as ProtoNet. The result is summarized in Table~\ref{app_tab:complexity}. Despite seemingly increased complexity in the training/test algorithms, our method incurs only constant-factor overhead compared to the minimal-cost ProtoNet. 

\begin{figure*}[t!]
\begin{center}
%
\centering
\includegraphics[trim = 1mm 2mm 5mm 4mm, clip, scale=0.285
]{figs/mini_memory_footprint.pdf}
%
\includegraphics[trim = 1mm 2mm 5mm 4mm, clip, scale=0.285
]{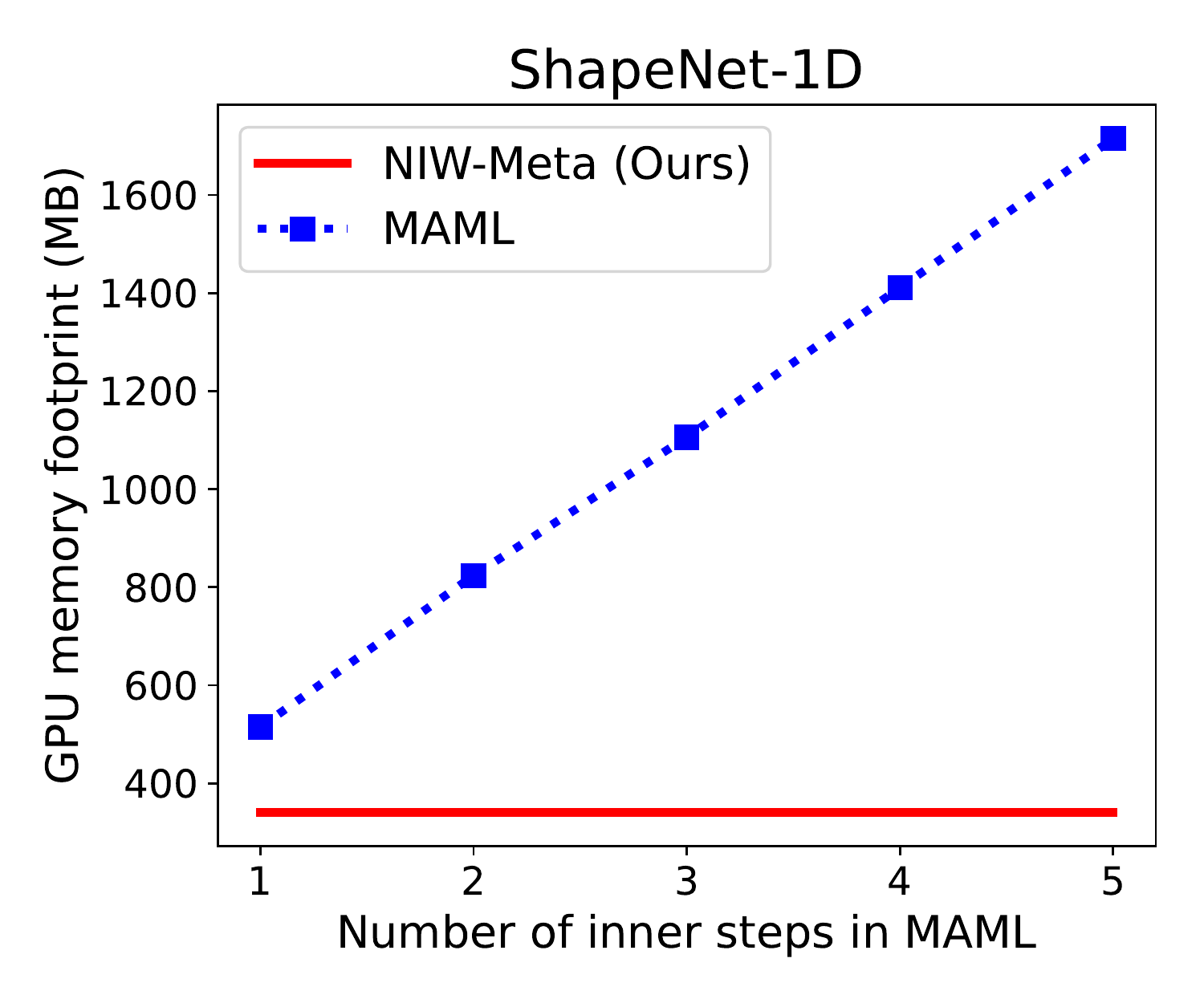}
\ \ \ \ \ \ 
\includegraphics[trim = 1mm 2mm 5mm 4mm, clip, scale=0.285
]{figs/mini_train_time.pdf}
%
\includegraphics[trim = 1mm 2mm 5mm 4mm, clip, scale=0.285
]{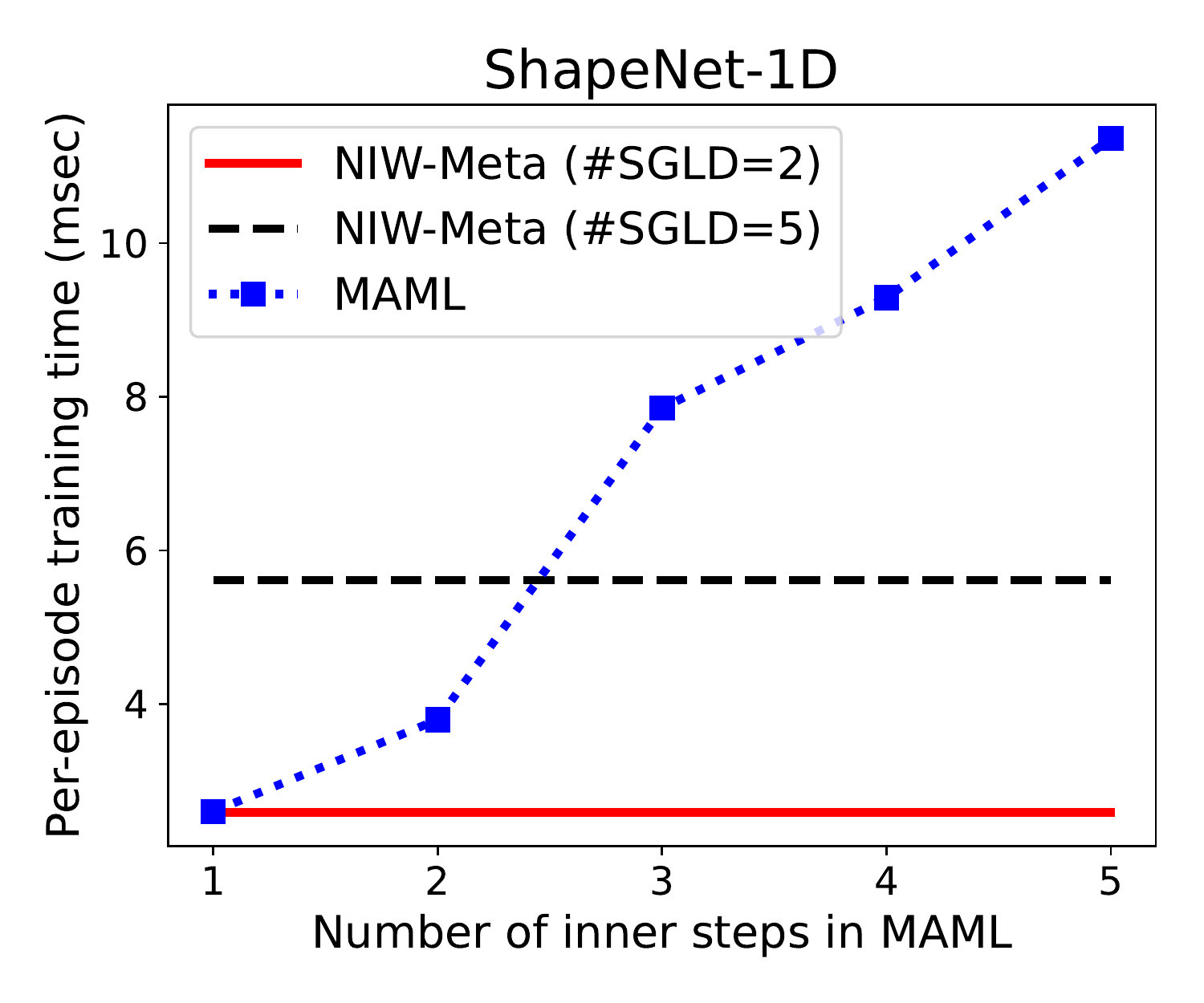}
\vspace{-1.5em}
\\ \ \ \ (a) GPU memory footprints \ \ \ \ \ \ \ \ \ \ \ \ \ \ \ \ \ \ \ \ \ \ \ \ \ \ \ \ \ \ \ \ \ \ \ \ \ \ \ \ \ \ \ \ \ \ \ \ \ \ (b) Per-episode training times
\end{center}
\vspace{-0.8em}
\caption{Computational complexity of MAML~\cite{maml} and our NIW-Meta. (a) GPU memory footprints (in MB) for a single batch. (b) Per-episode training times (in milliseconds). We use the ResNEt-18 backbone for {\em mini}ImageNet in 1-shot classification and the conv-net backbone for ShapeNet-1D regression (10 episodes per batch). 
}
\label{app_fig:complexity}
\end{figure*}

As we claimed in the main paper, one of the main drawbacks of MAML~\cite{maml} is the computational overhead to keep track of a large computational graph for inner gradient descent steps. Unlike MAML, our NIW-Meta has a much more efficient episodic optimization strategy, i.e., our one-time optimization only computes the (constant) first/second-order moment statistics of the episodic loss function without storing the full optimization trace. 

To verify this, we measure and compare the memory footprints and running times of MAML and NIW-Meta on two real-world classification/regression datasets: {\em mini}ImageNet 1-shot with the ResNet-18 backbone and ShapeNet-1D with the conv-net backbone. The results in Fig.~\ref{app_fig:complexity} show that NIW-Meta has far lower memory requirement than MAML (even smaller than 1-inner-step MAML) while MAML suffers from heavy use of memory space, nearly linearly  increasing as the number of inner steps. The running times of our NIW-Meta are not prohibitively larger compared to MAML where the main computational bottleneck is the SGLD iterations for quadratic approximation of the one-time episodic optimization. We tested two scenarios 
with the number of SGLD iterations 2 and 5, and we have nearly the same (or even better) training speed as the 1-inner-step MAML.

\end{document}